\documentclass[journal]{IEEEtran}
\usepackage{utils/utils}
\title{Fast and Distributed Equivariant Graph Neural Networks by Virtual Node Learning}

\author{Yuelin Zhang*$^{1\,2\,3}$, Jiacheng Cen*$^{1\,2\,3}$, Jiaqi Han$^4$, Wenbing Huang~\textsuperscript{\Letter}$^{1\,2\,3}$
\IEEEcompsocitemizethanks{
\IEEEcompsocthanksitem $*$ denotes equal contributions; \textsuperscript{\Letter} denotes the corresponding author: hwenbing@ruc.edu.cn.
\IEEEcompsocthanksitem $^1$ Gaoling School of Artificial Intelligence, Renmin University of China, Beijing 100872, China.
\IEEEcompsocthanksitem $^2$ Beijing Key Laboratory of Research on Large Models and Intelligent Governance, Beijing 100872, China.
\IEEEcompsocthanksitem $^3$ Engineering Research Center of Next-Generation Intelligent Search and Recommendation, MOE, Beijing 100872, China.
\IEEEcompsocthanksitem $^4$ Department of Computer Science, Stanford University, CA 94305, USA.
}%
\thanks{\copyright~2026 IEEE. Personal use of this material is permitted.
Permission from IEEE must be obtained for all other uses, in any current or
future media, including reprinting/republishing this material for advertising
or promotional purposes, creating new collective works, for resale or
redistribution to servers or lists, or reuse of any copyrighted component
of this work in other works.
\par
This is the author-accepted manuscript. The version of record is
published in \textit{IEEE Transactions on Pattern Analysis and Machine
Intelligence}, vol.~48, no.~9, pp.~10284--10298, September 2026.
\href{https://doi.org/10.1109/TPAMI.2026.3684569}{DOI 10.1109/TPAMI.2026.3684569}.%
}}

\begin{document}
\maketitle

\newif\ifrevision
\revisionfalse
\newcommand{\rev}[1]{%
  \ifrevision
    \textcolor{red}{#1}%
  \else
    #1%
  \fi
}

\newif\ifrevisiontwo
\revisiontwofalse
\newcommand{\revii}[1]{%
  \ifrevisiontwo
    \textcolor{red}{#1}%
  \else
    #1%
  \fi
}

\begin{abstract}

Equivariant Graph Neural Networks (GNNs) have achieved remarkable success across diverse scientific applications. However, existing approaches face critical efficiency challenges when scaling to large geometric graphs and suffer significant performance degradation when the input graphs are sparsified for computational tractability.
To address these limitations, we introduce FastEGNN and DistEGNN, two novel enhancements to equivariant GNNs for large-scale geometric graphs. FastEGNN employs a key innovation: a small ordered set of virtual nodes that effectively approximates the large unordered graph of real nodes. Specifically, we implement distinct message passing and aggregation mechanisms for different virtual nodes to ensure mutual distinctiveness, and minimize Maximum Mean Discrepancy (MMD) between virtual and real coordinates to achieve global distributedness. This design enables FastEGNN to maintain high accuracy while efficiently processing large-scale sparse graphs.
For extremely large-scale geometric graphs, we present DistEGNN, a distributed extension where virtual nodes act as global bridges between subgraphs in different devices, maintaining consistency while dramatically reducing memory and computational overhead.
We comprehensively evaluate our models across four challenging domains: N-body systems (100 nodes), protein dynamics (800 nodes), Water-3D (8,000 nodes), and our new Fluid113K benchmark (113,000 nodes). Results demonstrate superior efficiency and performance, establishing new capabilities in large-scale equivariant graph learning. Code is available at \url{https://github.com/GLAD-RUC/DistEGNN}.
\end{abstract}


\section{Introduction}

\IEEEPARstart{V}{arious} scientific data, such as chemical molecules, proteins, and particle-based physical systems, are often represented as geometric graphs~\cite{bronstein2021geometric}. \rev{Unlike standard graphs~\cite{gilmer2017neural,zhang2020learning,yuan2021explainability,zhang2022rethinking,liu2023learning,wang2022duality}}, geometric graphs additionally assign each node a set of geometric vectors---for example, 3D atomic coordinates in molecules. These graphs exhibit prominent symmetries, including translations, rotations, and reflections, stemming from the invariance of physical laws governing atomic (or particle) dynamics under changes in absolute position and orientation. When modeling such data, it is essential to incorporate these symmetries into the design, leading to the development of equivariant Graph Neural Networks (GNNs)~\cite{han2024survey,duval2023hitchhiker,zhang2025artificial}. In recent years, equivariant GNNs have achieved remarkable success in scientific applications, including physical dynamics simulation~\cite{wu2023equivariant,xu2024equivariant}, protein generation~\cite{watson2023novo,ingraham2023illuminating}, and many others.

Despite significant progress, existing equivariant GNNs face efficiency challenges when handling large geometric graphs, which are common in physical dynamics simulations. For example, EGNN~\cite{satorras2021en}, one of the most widely used models, incurs quadratic complexity in message exchange for fully connected graphs, making it impractical for large-scale scenarios like fluid dynamics simulations, where particle counts often exceed tens of thousands. A straightforward solution to improve efficiency is to reduce the number of edges, such as by restricting message passing to local neighbors within a cutoff radius. However, this approach risks losing long-range interactions and may degrade model performance.

\IEEEpubidadjcol

In this paper, we address the aforementioned challenge by introducing \emph{virtual nodes}. For clarity, we refer to the original nodes in geometric graphs as \emph{real nodes}. Each virtual node is designed to connect to all real nodes as well as other virtual nodes, enabling dense, global message passing. This preserves expressive power even when the underlying geometric graph is sparsified for efficiency. Drawing inspiration from the physics principle of equivalent particle interactions~\cite{darve2000fast}, we propose learning virtual nodes that serve as compact, representative proxies for groups of real nodes. To ensure effectiveness, these virtual nodes must satisfy two key geometric properties:
\emph{Mutual distinctiveness} encoding diverse structural features of the graph, and \emph{Global distributedness} maintaining spatial alignment with the real nodes' distribution.
These properties are interdependent and fundamental to the success of our method.

Specifically, we propose two efficient geometric GNN variants based on EGNN~\cite{satorras2021en}: FastEGNN and DistEGNN\footnote{This work is an extension of our previous conference version \cite{zhang2024improving}.}. In FastEGNN, we introduce an ordered set of virtual nodes, where each node assumes a distinct semantic role and undergoes independent message passing. This architecture explicitly enhances mutual distinctiveness among virtual nodes. Furthermore, we employ Maximum Mean Discrepancy (MMD)~\cite{borgwardt2006integrating} with an E(3)-invariant kernel to optimize the spatial alignment between virtual and real node coordinates, thereby improving global distributedness. Together, these innovations enable FastEGNN to efficiently simulate sparse graphs on a single device (\emph{e.g.}, GPU) without compromising performance.

To handle extremely large graphs (\emph{e.g.}, those exceeding 100K nodes) that surpass single-device capacity, we propose DistEGNN, a distributed extension of FastEGNN designed for multi-device parallel execution. Our approach employs a graph partitioning strategy to decompose the large graph into smaller subgraphs, which are distributed across multiple computing devices for parallel processing. While this partitioning improves scalability, it inherently introduces information loss across subgraph boundaries. To mitigate this, we implement a shared set of virtual nodes that span all devices, with their states continuously synchronized during training. This design serves two critical purposes: 1. Enabling global context propagation throughout the entire graph; 2. Facilitating cross-subgraph information fusion. 
By maintaining these global connections, DistEGNN effectively preserves the structural integrity of the complete graph while achieving efficient distributed computation, thereby minimizing the performance degradation typically associated with partitioned graph processing.

To sum up, our main contributions are as follows.
\begin{itemize}
    \item We propose FastEGNN upon EGNN~\cite{satorras2021en}, which introduces learnable virtual nodes to enable efficient global message passing on sparse large geometric graphs. By explicitly optimizing virtual nodes for mutual distinctiveness and global distributedness, our method preserves expressive power while maintaining computational efficiency. Moreover, we consider virtual node learning as a general plug-in and extend it to more backbones, including RF, SchNet, and TFN.  
    \item We further propose DistEGNN, a distributed extension of FastEGNN that breaks the memory barrier of single-device processing through innovative graph partitioning and synchronized virtual nodes learning. 
    Our approach maintains model accuracy while enabling distributed processing of large-scale geometric graphs. 
    \item We conduct comprehensive experiments to evaluate our proposed methods against State-Of-The-Art (SOTA) baselines across four challenging domains: N-body systems, Protein Dynamics, Water-3D, and our newly introduced Fluid113K dataset. To the best of our knowledge, this is the first work to perform equivariant graph learning on geometric graphs with over 100K nodes, significantly advancing the scalability frontier in this field. The results demonstrate both the efficiency and effectiveness of our models on such large-scale geometric graphs.
\end{itemize}

\section{Related Work}

\textbf{Geometric GNNs.}
Geometric GNNs can be classified into invariant and equivariant models. Invariant GNNs, exemplified by SchNet~\citep{schutt2018schnet} and DimeNet~\citep{Klicpera2020Directional}, have made initial attempts to embed geometric information into invariance features like distance and angles. However, a notable gap still exists in achieving full equivariance within these models, leading to the study of equivariant GNNs. For example, high-degree steerable GNNs, including TFN~\citep{thomas2018tensor}, SEGNN~\citep{brandstetter2022geometric}, and SE$(3)$-Transformer~\citep{fuchs2020se}, leverage the equivariance of spherical harmonics, offering commendable physical interpretability. Nevertheless, the substantial computational cost associated with spherical harmonic forms poses a significant limitation for their application in large-scale tasks.
In addition to the aforementioned models, scalarization-based GNNs~\citep{satorras2021en,kohler2019equivariant,jing2021learning,cen2024high} present a more elegant and expedient approach to representing geometric information, relying solely on linear combinations of geometric vectors.
Despite the rich literature and sophisticated architectures, the effectiveness of geometric GNNs usually relies on densely connected geometric graphs to ensure accuracy. However, this necessity introduces high complexity, thereby constraining the practical implementation of geometric GNNs in large-scale scientific problems. Our work tackles these challenges, aiming to enhance the scalability and applicability of geometric GNNs in scientific domains.

\textbf{Virtual nodes for GNNs.}
In prior investigations on conventional GNNs~\citep{gilmer2017neural,pham2017graph,hwang2022analysis}, attempts were made to maintain graph connectivity by introducing virtual nodes connected to all graph nodes. This design ensures that any two nodes can exchange information through the process of reading from and writing into virtual nodes at each step of message passing. In the realm of geometric GNNs, certain existing approaches treat virtual nodes as specifically defined clusters based on prior knowledge. They are exclusively connected to nodes belonging to the same cluster, such as different objects in SGNN~\citep{han2022learning} or distinct protein chains in MEAN~\citep{kong2023conditional}. In contrast to these methods, our approach involves end-to-end learning of virtual nodes from the dataset, eliminating the reliance on prior or external knowledge. Notably, the learning process for virtual nodes is thoughtfully designed to ensure both mutual distinctiveness and global distributedness.

\textbf{Scalability of GNNs.} In the field of conventional GNNs, extensive efforts have been devoted to overcoming the computational and memory bottlenecks of processing large-scale graphs. \rev{Some methods tackle this challenge by designing sampling strategies and compensatory mechanisms~\cite{chiang2019cluster,ding2024scalable,fey2021gnnautoscale, shi2023lmc, shi2025accurate}. For example, GNNAutoScale~\cite{fey2021gnnautoscale} achieves scalability by replacing the expensive computation of full neighbor trees with historical embeddings, keeping the memory usage constant even for massive graphs. While LMC~\cite{shi2023lmc} retrieves the messages discarded by subgraph-wise mini-batch sampling, enabling unbiased gradient estimates without needing to process the full computation graph.}
\rev{The TOP sampling method~\cite{shi2025accurate} leverages message invariance to approximate messages from out-of-batch neighbors via linear combinations of in-batch node embeddings, eliminating costly out-of-batch message passing and enabling fast and accurate mini-batch training on large graphs.}
Beyond sampling, some approaches have explored distributed training frameworks to further improve the scalability of GNNs. DistGNN~\cite{md2021distgnn}, for example, proposes a distributed training framework for 2D graphs. It integrates complex system-level mechanisms to reduce cross-device communication costs, but also increases the difficulty of implementation and deployment.
In contrast to these efforts, our proposed DistEGNN is the first distributed framework tailored specifically for geometric graphs with 3D spatial structures. We formulate a distributed message-passing scheme that preserves E$(3)$-equivariance, a critical property for modeling physical systems. To further improve scalability and global context sharing, we introduce a set of shared, ordered virtual nodes that facilitate efficient information propagation across devices. This design significantly accelerates training while maintaining high predictive accuracy.

\section{Preliminaries}

\begin{figure*}[!t]
    \centering
    \input{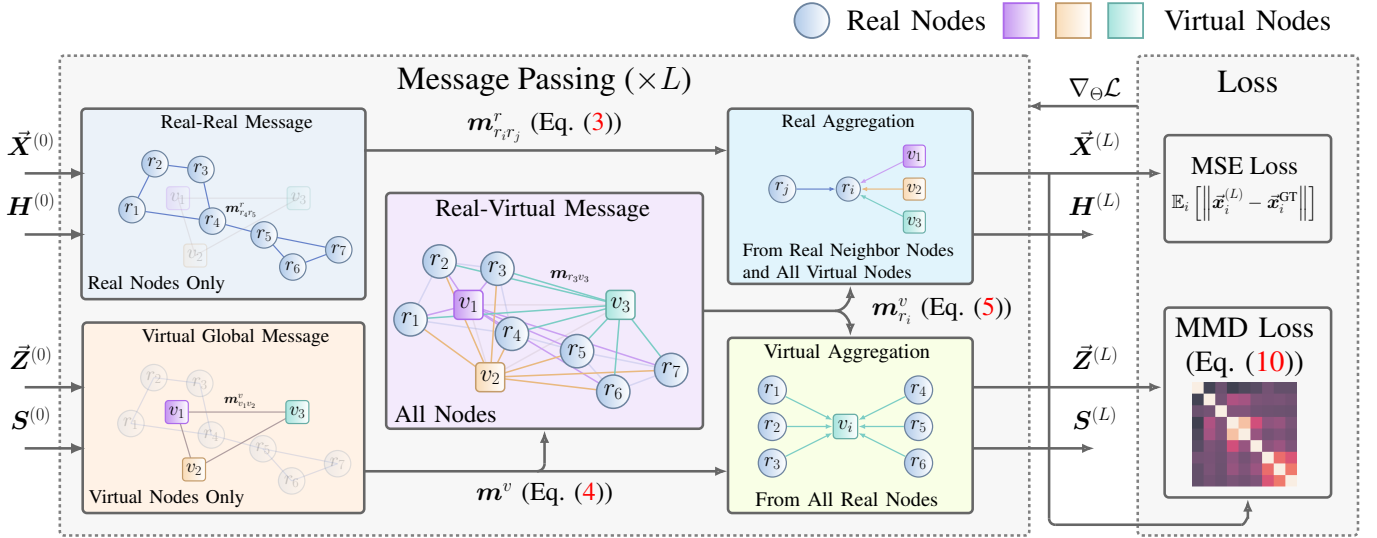}
    \vspace{-18pt}
    \caption{The overall architecture of FastEGNN. $(\gmX,\mH)$ are the real coordinates and features; $(\gmZ,\mS)$ are the virtual coordinates and features. Each layer contains 5 components: Real-Real Message $\vm^r_{r_ir_j}$, Virtual Global Message $\vm^v$, Real-Virtual Message $\vm^v_{r_i}$, Real Aggregation, and Virtual Aggregation. The real-real and real-virtual edges are displayed with different colors to indicate that the message passing and aggregation functions over the edges are different.}
    \label{fig:model}
\end{figure*}

\textbf{Geometric graph.} A geometric graph of $N$ nodes is defined as $\vec{\gG}\coloneqq(\gmX,\mH;\gE)$, where $\gmX=[\gvx_1,\gvx_2,\dots,\gvx_N]^\top\in\sR^{N\times 3}$ is the collection of 3D coordinates for the nodes, $\mH=[\vh_1,\vh_2,\dots,\vh_N]^\top\in\sR^{N\times H}$ is the node feature matrix, and $\gE$ is the set of edges representing the interactions between nodes. In the physical dynamics simulation scenarios, we are additionally provided with the velocities $\gmV=[\gvv_1,\gvv_2,\dots,\gvv_N]^\top\in\sR^{N\times 3}$ and optionally the edge feature $\ve_{ij}\in\sR^E$ for every edge $(i,j)\in\gE$.

\textbf{Equivariance.} Formally, a function $\phi$ is equivariant \emph{w.r.t.} group $G$ if $\phi(\rho_\gX(g) \vx)=\rho_\gY(g) \phi(\vx), \forall g\in G$, where $\rho_\gX(g)$ and $\rho_\gY(g)$
are the representation of transformation $g$ in the input space $\gX$ and output space $\gY$, respectively. In this work, we focus on the 3D Euclidean group E$(3)$, whose elements $g\in\text{E}(3)$ can be realized by an orthogonal matrix $\mO\in\sR^{3\time 3},\mO^\top\mO=\mI$ and a translation vector $\gvt\in\sR^3$. By this means, every E$(3)$-equivariant GNN  $\phi$ on geometric graph $\vec{\gG}$ taking as input coordinates $\gmX$, velocities $\gmV$, and node features $\mH$ should satisfy the following constraint for all $(\mO,\gvt\,)\in\text{E}(3)$:
\begin{align}
\label{eq:equivariance}
\gmX'\mO+\gvt,\mH'=\phi(\gmX\mO+\gvt,\gmV\mO,\mH,\gE),
\end{align}
where $\gmX',\mH'=\phi(\gmX,\gmV,\mH,\gE)$ are the updated coordinates and node features. Note that the velocities $\gmV$ are not affected by the translation $\gvt$. We do not specify permutation equivariance in \cref{eq:equivariance}, since it has been intrinsically encoded in GNNs.

\textbf{Problem formulation.} We mainly focus on the position prediction task which is fundamental to many physical simulation scenarios~\cite{fuchs2020se,satorras2021en}. Specifically, given the initial geometric graph with coordinates $\gmX$, node features $\mH$, velocities $\gmV$, edges $\gE$ and their attributes $\{\ve_{ij}:(i,j)\in\gE\}$, we seek to predict $\gmX'$ as the positions after a fixed time $\Delta t$. It is an E$(3)$-equivariant task, since if we rotate or translate the initial positions, the predicted positions should rotate or translate in accordance. That is, we aim to design an E$(3)$-equivariant function $\phi$ that meet the constraint of~\cref{eq:equivariance}.

\section{FastEGNN}
In this section, we introduce the details of our proposed FastEGNN. We first present the representation of virtual nodes in~\cref{sec:vr}. Then, we design the message passing and aggregation among the real and virtual nodes in~\cref{sec:message}, and provide the learning objectives of our model in~\cref{sec:loss}. Finally, we will explain the benefit of involving virtual nodes in large geometric graph learning in~\cref{sec:lgg}.

As claimed in Introduction, FastEGNN is designed to ensure the mutual distinctiveness and the global distributedness of the virtual nodes. We will specify that the mutual distinctiveness is guaranteed by the independent parameterization of the virtual nodes in~\cref{sec:vr}, and separate message passing and aggregation with different functions in~\cref{sec:message}. As for the global distributedness, we achieve this property by employing the MMD loss as a regularization term in~\cref{sec:loss}. 
\rev{The proofs of all propositions in this section are detailed in~\cref{appendix:proof}.}

\subsection{Geometric graphs with virtual nodes}
\label{sec:vr}

With the input geometric graph $\vec{\gG}$ as defined before, we create an augmented geometric graph $\vec{\gG}^v$ by adding a set of $C$ virtual nodes  $(\gmZ,\mS)$, where $\gmZ\in\sR^{3\times C}$ and $\mS\in\sR^{H\times C}$ refer to the 3D coordinates and invariant features, respectively. Each virtual node is connected to all real nodes in $\vec{\gG}$, as well as each other virtual node. This augments the original edge set $\gE$ with new connections, leading to $\gE^v$. Overall, $\vec{\gG}^v\coloneqq(\gmX,\gmV,\mH;\gmZ,\mS; \gE^v)$.

As the virtual nodes are not observed in the input data, we need to handcraft the initialization of their values. We adopt the following initialization strategy for the virtual nodes: $\gmZ=\frac{1}{N}\sum_{i=1}^N\gvx_i \bm{1}^\top$, where each channel in $\gmZ$ is initialized as the Center-of-Mass (CoM) of the input geometric graph, ensuring E$(3)$-equivariance; each channel of $\mS\in\sR^{H\times C}$ is initialized as independent learnable parameters that are optimized during training. 
In form, the virtual nodes are initialized as a function of the real nodes, namely, $\gmZ, \mS=\varphi_{\mathrm{init}}(\gmX,\mH)$, which should abide by:
\begin{align}
\label{eq:eqconstraint-vn}
\mO\gmZ+\gvt,\mS=\varphi_{\mathrm{init}}(\mP\gmX\mO+\gvt,\mP\mH),
\end{align}
for any orthogonal transformation $\mO\in\R^{3\times 3}$, translation $\gvt\in\R^3$, and permutation $\mP\in\{0,1\}^{N\times N}$. $\gmZ$ is E(3)-equivariant and permutation-invariant with respect to $\gmX$.  

One crucial point we would like to emphasize is that the channels of $\gmZ$ (and $\mS$) actually form an \emph{ordered} set; in other words, different channels, namely, different virtual nodes, are endowed with different meanings and can not be organized in arbitrary permutation. This restriction, along with the separate message passing mechanism in the next subsection, serves our purpose of attaining the mutual distinctiveness among the virtual nodes. Besides, such orderliness also facilitates the derivation of the message function with universal expressivity, which will be detailed in~\cref{sec:lgg}. 

\subsection{Message Passing and Aggregation}
\label{sec:message}

After obtaining the initialized virtual nodes, we conduct the update of all virtual nodes and all real nodes via E(3)-equivariant message passing, as illustrated in~\cref{fig:model}. This includes real-real message computation, virtual global message calculation, real-virtual message derivation, real node aggregation, and virtual node aggregation. We introduce each component in detail below. 

We denote by $(\gmX^{(0)},\gmV^{(0)},\mH^{(0)};\gmZ^{(0)},\mS^{(0)}; \gE^v)$ the initialization, and specify the $l$-th layer with superscript $l$.

\textbf{Real-Real message.}
We first derive the message between all the real nodes in the way akin to EGNN~\citep{satorras2021en}, which is given by:
\begin{align}\label{eq:mij}
    \vm^r_{ij}=\varphi_1(\vh_i^{(l)},\vh_j^{(l)},\|\gvx_i^{(l)}-\gvx_j^{(l)}\|^2,\ve_{ij}),
\end{align}
where $\varphi_1$ is a Multi-Layer Perceptron (MLP). 

\textbf{Virtual global message.}
Each virtual node represents a distinct aspect of the geometric graph. Therefore, exhaustively modeling the correlations between each pair of virtual nodes enables the model to reflect the global behavior of the input system. To do so, we compute an invariant matrix extracted from the inner-product of the center-translated virtual coordinates:
\begin{align}\label{eq:mv}
    \vm^v=(\gmZ^{(l)}-\bar\vx\bm{1}^\top)^\top(\gmZ^{(l)}-\bar\vx\bm{1}^\top),
\end{align}
where $\bar\vx=\frac{1}{N}\sum_{i=1}^N\gvx_i$ is the CoM of the geometric graph. Clearly, $\vm^v\in\sR^{C\times C}$ is an E$(3)$-invariant matrix, and indeed a universal approximator for E(3)-invariant functions of the virtual coordinates $\gmZ^{(l)}$, according to the proof of Lemma 2 in~\cite{huang2022equivariant}.

\textbf{Real-Virtual message.}
The next step is to compute the message between the real nodes and the virtual nodes by:
\begin{align}\label{eq:miv}
\vm^v_{i}=\varphi_2\left(\vh_i^{(l)},\mS^{(l)},\bigoplus_{c=1}^C \|\gvx_i^{(l)}-\gvz_c^{(l)}\|^2,\vm^v\right),
\end{align}
where $\bigoplus$ defines the concatenations along the channel dimension, and $\varphi_2$ is an MLP. \cref{eq:miv} returns the message between each real node $i$ and the global virtual set, which will be consistently applied during the later real and virtual node aggregation processes. Again, \cref{eq:miv} is E(3)-invariant. In our implementation, we find that computing message between real node $i$ and virtual node $c$ separately, namely,
$\vm_{ic}^{v}=\varphi_2\left(\vh_i^{(l)},\vs_c^{(l)}, \|\gvx_i^{(l)}-\gvz_c^{(l)}\|^2,\vm_c^v\right)$ enables more effective training and better performance, in contrast to the global message from all virtual nodes in~\cref{eq:miv}. Here, $\vm_c^v$ denotes the $c$-th column of $\vm^v$.

\textbf{Real aggregation.}
Given the real-real message $\vm^r_{ij}$ and real-virtual message $\vm^v_i$, we are ready to derive the message aggregation for real node $i$ as follows:
\begin{align}
\gvx_i^{(l+1)}=
\nonumber
&\gvx_i^{(l)}+\alpha_i\sum_{j\in\gN(i)}(\gvx_i^{(l)}-\gvx_j^{(l)})\varphi^{r}_x(\vm^r_{ij})+\\
\label{eq:xr(l+1)}
& \frac{1}{C}\sum_{c=1}^C (\gvx_i^{(l)}-\gvz_c^{(l)})\varphi^{v}_{x}(\vm_{ic}^{v})+\varphi_v(\vh_i^{(l)})\gvv_i^{(0)},\\
\vh_i^{(l+1)}=&\vh_i^{(l)}+\varphi_h\left(\vh_i^{(l)},\alpha_i\sum_{j\in\gN(i)}\vm^r_{ij},\bigoplus_{c=1}^C\vm_{ic}^{v},\right),
\label{eq:hr(l+1)}
\end{align}
where $\alpha_i=\frac{1}{|\gN(i)|}, \varphi^r_x, \varphi^v_{x}, \varphi_v$ are MLPs outputting one-dimensional scalars, and $\varphi_h$ is also an MLP with the output dimension complying with that of $\vh_i^{(l+1)}$. The update of $\gvx_i^{(l+1)}$ (and similarly $\vh_i^{(l+1)}$) is inspired by EGNN but further extended with the message from the virtual nodes, namely, the term $\sum_{c=1}^C (\gvx_i^{(l)}-\gvz_j^{(l)})\varphi^v_{x}(\vm^v_{i})$. Previous idea of adding a global node in SGNN~\cite{han2022learning} is more like a degenerated version of ~\cref{eq:xr(l+1)} where $\varphi^r_{x}$ and $\varphi^v_{x}$ are shared. It is important to note that as the number of channels $C$ increases, the concatenation form $\bigoplus_{c=1}^C \vm_{ic}^{v} \in \sR^{CH}$ in \cref{eq:hr(l+1)} results in a significant increase in the number of MLP parameters. To enhance the scalability, we instead adopt the summation form $\beta \sum_{c=1}^C \vm_{ic}^{v}$, where $\beta = \frac{1}{C}$.

\textbf{Virtual aggregation.}
Now, we devise the virtual node aggregation as:
\begin{align}
\label{eq:Xv(l+1)}
\gvz_c^{(l+1)}&=\gvz_c^{(l)}+\frac{1}{N}\sum_{i=1}^N (\gvz_c^{(l)}-\gvx_i^{(l)})\varphi_{Z}(\vm_{ic}^{v}),\\
\label{eq:Hv(l+1)}
\vs_c^{(l+1)}&=\vs_c^{(l)}+\varphi_{S}\left(\vs_c^{(l)},\frac{1}{N}\sum_{i=1}^N\vm_{ic}^{v} \right),
\end{align}
where, $\varphi_{Z}\in\R$ and $\varphi_{S}$ are also MLPs. To be specific, the calculation $\sum_{i=1}^N (\gvz_c^{(l)}-\gvx_i^{(l)})\varphi_{Z}(\vm_{ic}^{v})$ in~\cref{eq:Xv(l+1)} aggregates all messages from all real nodes, multiplied with different scalar $\varphi_{Z}(\vm_{ic}^{v})$. In this way, the update of different virtual node is independent to each other, for the sake of encouraging better mutual distinctiveness.

We have the flowing theoretical assurance by our design:
\begin{proposition}\label{prop:equofnn}
If the initialization of the virtual nodes satisfies~\cref{eq:eqconstraint-vn}, then after~\cref{eq:mij,eq:mv,eq:miv,eq:xr(l+1),eq:hr(l+1),eq:Xv(l+1),eq:Hv(l+1)}, the output coordinates $\gmX^{(L)}$ are E(3)-equivariant and permutation-equivariant, the virtual coordinates $\gmZ^{(L)}$ are E(3)-equivariant and permutation-invariant, with respect to the input $\gmX^{(0)}$.
\end{proposition}

\subsection{Learning Objectives}
\label{sec:loss}

In this subsection, we exploit the MMD objective~\cite{borgwardt2006integrating} to
enforce the alignment between virtual and real coordinates. MMD is widely used in the research of domain adaption. Here, we derive $\gL_\mathrm{MMD}$ without the original term $k(\gvx_i^{(L)}, \gvx_j^{(L)})$ as:
\begin{align}\label{eq:mmd}
     \frac{1}{C^2}\sum_{i=1}^C\sum_{j=1}^C k(\gvz_i^{(L)}, \gvz_j^{(L)})-\frac{1}{NC}\sum_{i=1}^N\sum_{j=1}^C k (\gvx_i^{(L)},\gvz_j^{(L)}),
\end{align}
where the RBF kernel $k(\gvx,\gvy)=\exp(-\frac{\|\gvx-\gvy\|^2}{2\sigma^2})$ is E(3)-invariant, hence the MMD loss is also E(3)-invariant. Interestingly, minimizing the first term in~\cref{eq:mmd} is actually enlarging the divergence between the virtual nodes, while maximizing the second term enhances the similarity between the virtual nodes and the real ones. Note that we can just sample a small-scale subset of the real nodes for the MMD calculation at each iteration to avoid redundant costs. As the MMD is evaluated exclusively during training, this sampling strategy will not compromise the model’s equivariance or invariance.

The overall training loss $\gL$ is a combination of the MSE loss between the predicted positions and the ground truth, together with the proposed auxiliary E$(3)$-invariant MMD loss to expand the coverage of the virtual nodes towards the data distribution:
\begin{align}\label{eq:loss}
    \gL=\gL_\mathrm{MSE}(\gmX^{(L)},\gmX^\mathrm{GT})+\lambda\gL_\mathrm{MMD}(\gmZ^{(L)},\gmX^\mathrm{GT}),
\end{align}
where $\lambda$ is the balancing factor between MSE and MMD.

\subsection{Theoretical Analyses}
\label{sec:lgg}

The main complexity of FastEGNN lies in the aggregation processes for each node in \cref{eq:xr(l+1),eq:hr(l+1)}, where the number of the summation operations for all real nodes is $NK+NC$, with $K$ denoting neighbor size. To improve the efficiency for large graphs, one solution is decreasing the value of $K$, which yet will diminish the interaction between the real nodes. In this subsection, we will explain that such diminution, owing to the involvement of the virtual nodes, is alleviated and our FastEGNN will still perform well even when $K$ is decreased to zero, if the distribution of the virtual coordinates is well aligned with the real coordinates.

We focus on the update of the real coordinate $\gvx_i$ by aggregating messages from its neighbors. Without loss of generality, we assume its neighbors to be the whole set $\gmX$. The update is denoted as the function $\gvx'_i=f(\gvx_i, \gmX)$, which should be E(3)-equivariant \emph{w.r.t.} the both inputs, and permutation-invariant \emph{w.r.t.} $\gmX$. According to Proposition 10 in~\cite{villar2021scalars}, we have the following result.
\begin{proposition}[\cite{villar2021scalars}]
\label{prop:xX}
The update function must take the form $f(\gvx_i, \gmX)=\gvx_i+\sum_{j=1}^N (\gvx_j-\gvx_i)\psi(\gvx_j-\gvx_i, \gvx_1-\gvx_i,\cdots,\gvx_{j-1}-\gvx_i,\gvx_{j+1}-\gvx_i,\cdots,\gvx_{N}-\gvx_i)$, where $\psi:\R^{3N}\rightarrow\R$ is orthogonality-invariant, and permutation-invariant with respect to the last $N-1$ inputs.
\end{proposition}

While this result is theoretically elegant, it is not so practically helpful, since the universal form of $\psi$ is still unknown. Existing equivariant GNNs (such as EGNN) only exploit its reduced but not sufficiently expressive version. 

In our model, we need to aggregate the message from the virtual nodes $\gmZ$. The update function becomes $\gvx'_i=f(\gvx_i, \gmZ)$. Notably, different from $\gmX$, the matrix $\gmZ$ is an ordered set by our design, which entails that the permutation invariance of $f$ is no longer required. This enables us to derive a more informative form of $f$ as follows.
\begin{proposition}
\label{prop:xZ}
The update function must take the form $f(\gvx_i, \gmZ)=\gvx_i+\sum_{c=1}^C (\gvz_c-\gvx_i)\psi_c\left(\bigoplus_{c=1}^C \|\gvz_c-\gvx_i\|^2, \vm^v\right)$, where $\psi_c:\R^{C+C^2}\rightarrow\R$ is an arbitrary non-linear function, and $\vm^v$ is an E(3)-invariant term by~\cref{eq:miv}. 
\end{proposition}

The proof is based on Proposition 1 in SGNN~\cite{han2022learning}, and provided in appendix. \cref{prop:xZ} is practically realizable by just implementing $\psi_c$ via MLP, which is actually the form in our FastEGNN (\cref{eq:miv}). More importantly, \cref{prop:xZ} suggests that $f(\gvx_i, \gmZ)$ is able to universally approximate the messages from all the real nodes, if the virtual coordinates $\gmZ$ can well approximate the real ones $\gmX$. As presented in the last subjection, we achieve this goal by penalizing the MMD loss. Our latter experiments will also support the theoretical analyses here, showing that FastEGNN with a small value of $C$ can still perform promisingly even when all edges are dropped, while EGNN behaves poorly. 

\subsection{Generalization to Other Backbones}

\rev{Our proposed virtual node mechanism is modular and flexible, designed as a plug-in module that can be easily integrated into a wide range of geometric graph neural networks.
In FastEGNN, we keep the original EGNN architecture intact, preserving the message exchange and feature update rules for real nodes. On top of this, we introduce an auxiliary message-passing pathway between real nodes and the newly added virtual nodes. This interaction is managed by a separate, parameterized module with its own update rules, ensuring that our mechanism can be seamlessly applied to other geometric GNN architectures—extending its benefits without requiring changes to their core designs.}

\rev{To examine its general applicability, we apply the mechanism to three representative geometric GNNs: RF~\citep{kohler2019equivariant}, an equivariant model; SchNet~\citep{schutt2018schnet}, an invariant GNN; and TFN~\citep{thomas2018tensor}, a high-degree steerable model that leverages spherical harmonics for equivariant representation. The implementation details and experiment results are detailed in~\cref{appendix:generality_analysis}.}

\section{DistEGNN}
\label{sec:distributed_fastegnn}

Distributed training can significantly improve the scalability and computational efficiency of geometric GNNs. By distributing the data and computation overhead of large-scale geometric graphs across multiple devices, it not only accelerates the training process but also enables the model to handle extremely large graph structures that are otherwise infeasible on a single device. This capability is particularly important for simulating real-world systems with complex and mass spatial relationships.

Taking advantage of the global communication provided by virtual nodes, we develop a distributed equivariant GNN, called DistEGNN, which not only improves computational efficiency but also maintains strong predictive performance. 
Our method fuses both data parallelism and model parallelism. On one hand, a large graph is partitioned into multiple subgraphs, each processed independently by a model replica on a separate device. On the other hand, we use virtual nodes as communication intermediaries across devices, and the model’s computation inherently depends on the flow of information between the devices.

Specifically, as illustrated in~\cref{fig:distribute_framework}, for an extremely large graph that cannot fit into the memory of a single device, we first use a graph partition method to map each node to a distinct computing device.
Subsequently, each device constructs its own local graph based on the nodes distributed to it.
Besides the real nodes, we also initialize a shared, ordered set of virtual nodes on each local graph to enable cross-device message passing. 
\rev{We still initialize the positions of the virtual nodes at the CoM of the entire large graph in DistEGNN. In this way, we ensure the equivariance as shown in~\cref{prop:equi_distegnn_init}. 
\begin{proposition}\label{prop:equi_distegnn_init}
The CoM initialization of DistEGNN is $E(3)$-equivariant.
\end{proposition}}

During the training process, each device independently conducts parallel message passing on its assigned local graph. Meanwhile, the virtual nodes are kept synchronizing across devices in real time, ensuring seamless propagation of global information across partitions. Formally, the computations of real-real message, virtual global message, and real-virtual message follow~\cref{eq:mij,eq:mv,eq:miv}, while the real aggregation process adheres to~\cref{eq:xr(l+1),eq:hr(l+1)}. These operations are identical to those in the single-device version and can be executed in parallel on each device for the corresponding local subgraphs. By contrast, updating the virtual nodes in the distributed version requires aggregating messages from all real nodes across all devices. To formalize this, let $D$ denote the total number of the used devices, and $n_d$ indicates the index of each real node in the $d$-th device. The update for the virtual nodes in the distributed scenario is given by:
\begin{equation}\label{eq:upd_z_dist}
\gvz_c^{(l+1)} = \gvz_c^{(l)} + \frac{1}{N} \sum_{d=1}^{D}\sum_{n_d=1}^{N_{d}} (\gvz_c^{(l)}-\gvx_{n_d}^{(l)}) \varphi_{Z}(\vm_{n_dc}^{v}),
\end{equation}
\begin{equation}\label{eq:upd_s_dist}
\vs_c^{(l+1)}=\vs_c^{(l)}+\varphi_{S}\left(\vs_c^{(l)},\frac{1}{N} \sum_{d=1}^{D}\sum_{n_d=1}^{N_{d}}\vm_{n_dc}^{v} \right),
\end{equation}
where the first summation in \cref{eq:upd_s_dist} denotes aggregation over devices, indexed by $d=1,\dots,D$, and requires communication between devices.

The learning objective in~\cref{eq:loss} is distributable across multiple devices. To be specific, we compute the loss of the $d$-th device in parallel as follows:
\begin{align} \label{eq:dist_loss_func}
    \gL_d=\gL_\mathrm{MSE}(\gmX^{(L)}_{d},\gmX^{\text{GT}}_{d})+\lambda\gL_\mathrm{MMD}(\gmZ^{(L)}, \gmX^{\text{GT}}_{d}),
\end{align}
where $\gmX^{(L)}_{d}$ and $\gmX^{\text{GT}}_{d}$  respectively denote the predicted and the ground-truth node coordinates in the $d$-th device. 

\begin{figure}[ht!]
    \centering
    \includegraphics[width=\linewidth]{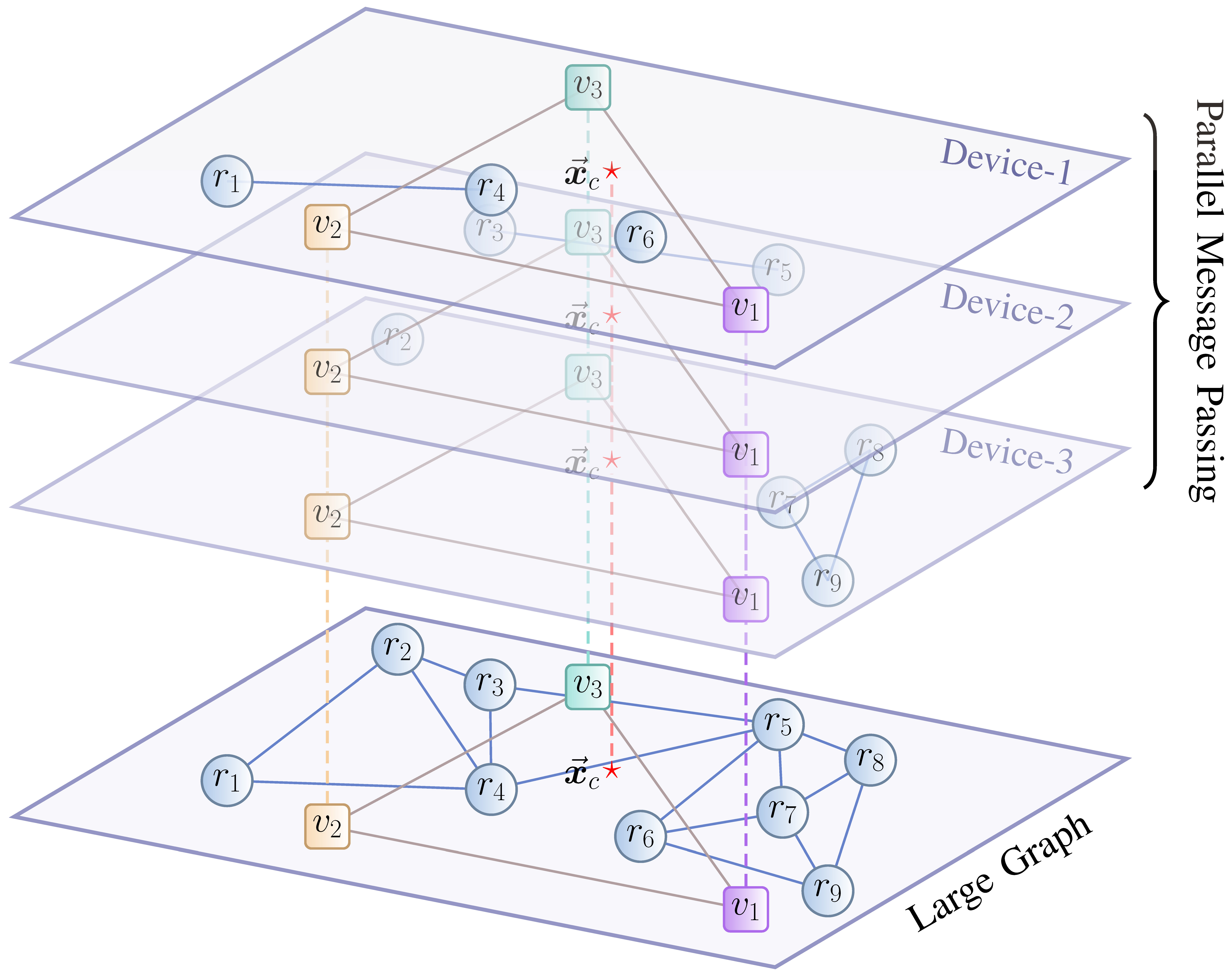}
    \vspace{-18pt}
    \caption{Overview of DistEGNN. Multiple devices collaboratively process a large graph in parallel. The core of our design is the shared virtual nodes across devices, which propagate global information and mitigate the performance loss caused by graph partition.}
    \label{fig:distribute_framework}
\end{figure}

\begin{algorithm}
\caption{Training DistEGNN on $D$ Devices.}
\label{alg:dist}
\KwData{Large graph $\vec{\gG} = (\gmX,\mH;\gE)$.}
\KwResult{Trained DistEGNN model with $L$ layers.}
\tcp{Initialization.}
Initialize virtual nodes coordinates $\gmZ=\frac{1}{N}\sum_{i=1}^N\gvx_i \bm{1}^\top$, learnable features $\mS \in \sR^{H\times C}$\;
Partition $\vec{\gG}$ into $D$ parts and construct subgraphs with the shared virtual nodes, assigning $\vec{\gG}^v_{1}, \dots, \vec{\gG}^v_{D}$ to $D$ devices\;

\tcp{Parallel Message Passing.}
\For{$l = 1$ to $L$}{
    CoM of graph $\bar\vx = \frac{1}{N} \sum_{d=1}^{D} \left(\sum_{n_d=1}^{N_{d}} \gvx_{n_d}\right)$\;
    Calculate message using \cref{eq:mij,eq:mv,eq:miv}\;
    \tcp{For Real Nodes.}
    Aggregate and update with \cref{eq:xr(l+1),eq:hr(l+1)}\;
    \tcp{For Virtual Nodes.}
    Aggregate and update with \cref{eq:upd_z_dist,eq:upd_s_dist}\;
}

\tcp{Calculate loss and update model.}
Calculate local loss $\mathcal{L}_d$ with \cref{eq:dist_loss_func} and compute gradients $\nabla_\theta \mathcal{L}_d$ on each device\;
Synchronize gradients across devices to get $\nabla_\theta \mathcal{L}$\;
Update model parameters on each device: $\theta \leftarrow \theta - \eta \nabla_\theta \mathcal{L}$.
\end{algorithm}

\rev{\textbf{Cross-Device Aggregation Operation.}} PyTorch provides a function defined as \texttt{all\_reduce}  to aggregate information across devices. However, it is important to note that, in its implementation, gradients do not propagate back through the collective communication path to each individual device. This behavior leads to disrupted gradient flow in scenarios where the result of the reduction operation participates in subsequent gradient computations in~\cref{eq:upd_z_dist,eq:upd_s_dist}. To address this issue, we implement a custom cross-device aggregation operation with explicitly defined forward and backward behaviors. In the forward pass, the operation performs distributed summation across devices, while in the backward pass, it ensures gradients are correctly redistributed and propagated back to all contributing devices. 
This design helps to preserve full computational graph in the distributed environment. During training, each device computes the loss independently based on its local subgraph, while our customized method guarantees correct gradient backpropagation across distributed devices. For inference, since gradient computation is unnecessary, we disable gradient tracking for improved efficiency.

We present the pseudocode for training DistEGNN in~\cref{alg:dist}. \rev{\cref{prop:equi_distegnn} ensures the equivariance of our DistEGNN. 
We further design a new variant of DistEGNN that incorporates a device-level attention mechanism, assigning distinct attention weights from the virtual node to each device.
Details are provided in~\cref{appendix:deivce_attention}.
\begin{proposition}\label{prop:equi_distegnn}
The message passing process defined in \cref{alg:dist} of DistEGNN is $E(3)$-equivariant.
\end{proposition}}

\section{Experiments}

In this section, we conduct comprehensive experiments to evaluate the \textbf{effectiveness} of FastEGNN, and the \textbf{scalability} of DistEGNN on challenging physical simulation tasks over large geometric graphs. 
We begin by introducing the datasets used in our experiments in~\cref{sec:dataset}.
Then, in~\cref{sec:exp_effectiveness}, we demonstrate the effectiveness of our proposed FastEGNN through benchmark comparisons and ablation studies.
Finally, we evaluate the efficiency and efficacy of the DistEGNN on Water-3D and our generated extremely large graph dataset Fluid113K in~\cref{sec:exp_scalability}.

\subsection{Dataset Details}
\label{sec:dataset}

We comprehensively benchmark FastEGNN on three challenging simulation datasets to evaluate its effectiveness and generalizability. Additionally, we introduce a fourth dataset of custom-generated large-scale graphs (with 113K nodes per graph on average), designed to assess the scalability of DistEGNN.

\textbf{$N$-body System}~\cite{satorras2021en,kipf2018neural}. In this simulation, each system comprises $N=100$ charged particles with random charge $c_i\in\{\pm 1\}$, whose movements are driven by coulomb force. The graph is constructed in a fully-connected manner. We use 5,000 samples for training, 2,000 for validation, and 2,000 for testing. The task is to predict the particle positions at frame 40 given their positions and velocities at frame 30.

\textbf{Protein Dynamics}~\cite{han2022equivariant}. The protein molecular dynamics dataset is processed from MDAnalysis~\cite{gowers2016mdanalysis}, which depicts a long-range AdK equilibrium MD trajectory~\cite{seyler5108170molecular}. Following previous work~\cite{han2022equivariant}, we model the dynamics of the backbone atoms, leading to 855 nodes per sample. A total number of 55,108 edges (on average) are connected between the atoms within a distance cutoff of 10\AA. The dataset has been split into train/validation/test sets that contain 2,481/827/878 frame pairs, respectively, with the time span $\Delta t=15$.

\textbf{Water-3D}~\cite{sanchez2020learning}. Water-3D is a large-scale particle-based fluid simulation dataset generated with Smoothed-Particle Hydrodynamics (SPH). The dataset records the dynamics of water falling in a box, with 1000 trajectories for training, 100 for validation, and 100 for testing. There are on average a large amount of 7,806 particles and 94,931 edges in each system, where the edges are connected with a cutoff of 0.035. The prediction task involves forecasting particle positions with a time span of $\Delta t=15$ frames. It serves as a challenging benchmark for FastEGNN.

\textbf{Fluid113K}~\cite{Ummenhofer2020Lagrangian}. Fluid113K is an extremely large-scale particle-based fluid simulation dataset constructed based on the open source implementation from DeepLagrangianFluids~\cite{Ummenhofer2020Lagrangian}. The simulation is performed within the SPlisHSPlasH~\cite{SPlisHSPlasHLibrary} framework. We adopt the scenario in which a single fluid body falls inside a standard cubic container. To increase the scale of each simulation, we enlarge both the particle size and container size, ensuring that each trajectory contains more than 100,000 particles. In total, we generate 100 trajectories for training, 20 for validation, and 20 for testing. Each trajectory simulates 4 seconds of motion, recorded at 50 Hz, resulting in 200 frames. During training, we randomly sample 16 starting time steps from each trajectory and set the prediction interval to $\Delta t = 20$. On average, each frame contains 113,140 particles and 1,706,973 edges. As an exceptionally large graph dataset, Fluid113K poses significant challenges for evaluating the scalability of geometric graph neural networks. The data generation process may take approximately a week to complete, depending on the CPU performance.

\subsection{Effectiveness of FastEGNN} 
\label{sec:exp_effectiveness}

\textbf{Implementation Details.} To evaluate the effectiveness of our method, we instantiate several variants of FastEGNN by exploring different combinations of the number of virtual nodes $C$ and the edge dropping rate $p$, denoted as FastEGNN-$\langle C, p \rangle$. We adopt the following edge dropping strategy: We sort all edges based on the distance between the connected nodes $\|\vec\vx_i-\vec\vx_j \|_2$ and drop the top $p\,(\times 100\%)$ longest edges.  

\textbf{Baselines.} We compare our models with the following baselines: the simplest equivariant model Linear dynamics~\cite{satorras2021en}, the non-equivariant Message-Passing Neural Network (MPNN)~\cite{gilmer2017neural}, the invariant GNN model SchNet~\cite{schutt2018schnet}, and the equivariant GNNs including Tensor Field Networks (TFN)~\cite{thomas2018tensor}, Radial Field (RF)~\cite{kohler2019equivariant}, and EGNN~\cite{satorras2021en}. We also evaluate EGNN$^\ast$, a variant of EGNN that removes all edges in the graph, as a reference.

\textbf{Metrics.} \textbf{1.} \emph{MSE}: We use the Mean Squared Error (MSE) between the predicted position and the ground truth on the testing set as the metric to measure the prediction accuracy. Note that we add random rotations on input and target graph during testing to assess the model's equivariance.
\textbf{2.} \emph{Relative Time}: To demonstrate the speed-up effect of FastEGNN, we also benchmark the inference time of all models to traverse through the entire testing set, and compute their relative scales \emph{w.r.t.} the inference time of EGNN.

\subsubsection{Main Results}

\begin{table*}[ht!]
\centering
\caption{Results of FastEGNN and baselines on $N$-body System ($100$ nodes), Protein Dynamics ($855$ nodes), and Water-3D ($7806$ nodes) datasets. FastEGNN-$\langle C, p\rangle$ denotes the model with the number of virtual nodes as $C$ and edge dropping rate as $p$. We also report the results of EGNN* which indicates the EGNN model with all edges dropped.}
\label{tab:main_result}
\adjustbox{width=\textwidth, center}{
\begin{tabular}{
    p{0.18\textwidth}<{\raggedright}
    p{0.12\textwidth}<{\centering}
    p{0.12\textwidth}<{\centering}
    p{0.12\textwidth}<{\centering}
    p{0.12\textwidth}<{\centering}
    p{0.12\textwidth}<{\centering}
    p{0.12\textwidth}<{\centering}
}
\toprule
    & \multicolumn{2}{c}{$N$-body System} & \multicolumn{2}{c}{Protein Dynamics} & \multicolumn{2}{c}{Water-3D} \\
    & MSE ($\times 10^{-2}$) & Relative Time & MSE & Relative Time & MSE ($\times 10^{-4}$) & Relative Time \\
\midrule
    Linear                               & $12.66$ & $0.01$  &  $2.26$   & $0.01$  &  $14.06$   & $0.01$  \\
    MPNN~\citep{gilmer2017neural}        & $3.06$  & $0.62$  &  $150.56$ & $0.62$  &  $5299.30$ & $0.61$  \\
    SchNet~\citep{schutt2018schnet}      & $24.83$ & $1.39$  &  $2.56$   & $1.40$  &  $35.02$   & $2.44$  \\
    RF~\citep{kohler2019equivariant}     & $5.76$  & $0.26$  &  $2.25$   & $0.25$  &  $12.94$   & $0.25$  \\
    TFN~\citep{thomas2018tensor}         & $1.62$  & $17.02$ &  $2.26$   & $17.81$ &  $9.10$    & $19.53$ \\
    EGNN~\citep{satorras2021en}          & $1.41$  & $1.00$  &  $2.25$   & $1.00$  &  $6.00$    & $1.00$  \\
    EGNN*~\citep{satorras2021en}         & $11.62$ & $0.03$  &  $2.26$   & $0.06$  &  $12.38$   & $0.11$  \\ 

\midrule
    FastEGNN-$\langle 1 ,\ 0.00\rangle$  & $1.04$          & $1.15$  &  $1.89$          & $1.15$  &  $3.25$          & $1.35$ \\
    FastEGNN-$\langle 1 ,\ 0.75\rangle$  & $1.03$          & $0.32$  &  $1.91$          & $0.65$  &  $3.64$          & $0.47$ \\
    FastEGNN-$\langle 1 ,\ 1.00\rangle$  & $9.72$          & $0.07$  &  $2.00$          & $0.12$  &  $4.36$          & $0.22$ \\

\midrule
    FastEGNN-$\langle 3 ,\ 0.00\rangle$  & $1.06$          & $1.19$  &  $\textbf{1.82}$ & $1.23$  &  $\textbf{2.58}$ & $1.51$ \\
    FastEGNN-$\langle 3 ,\ 0.75\rangle$  & $1.09$          & $0.36$  &  $1.85$          & $0.72$  &  $3.13$          & $0.65$ \\
    FastEGNN-$\langle 3 ,\ 1.00\rangle$  & $9.52$          & $0.11$  &  $1.92$          & $0.18$  &  $3.60$          & $0.40$ \\

\midrule
    FastEGNN-$\langle 10,\ 0.00\rangle$  & $1.10$          & $1.36$  &  $1.84$          & $1.47$  &  $2.66$          & $1.81$ \\
    FastEGNN-$\langle 10,\ 0.75\rangle$  & $\textbf{0.99}$ & $0.53$  &  $1.88$          & $0.98$  &  $2.94$          & $0.85$ \\
    FastEGNN-$\langle 10,\ 1.00\rangle$  & $9.25$          & $0.27$  &  $1.99$          & $0.43$  &  $3.40$          & $0.61$ \\
    
\bottomrule
\end{tabular}
}
\end{table*}

The main quantitative results are presented in~\cref{tab:main_result}. We have the following observations:

\textbf{1.} Our FastEGNN yields the lowest simulation error on all three benchmarks, consistently outperforming the competitive baselines by a significant margin. For instance, FastEGNN yields 29\% and 19\% improvement in terms of MSE over the best-performed baseline EGNN on $N$-body system and Protein Dynamic datasets, respectively. On the most challenging dataset Water-3D with an average of 7,806 particles in each system, FastEGNN reaches a remarkably low simulation error of $2.58\times 10^{-4}$, as opposed to EGNN with an MSE of $6.00\times 10^{-4}$. The strong results unanimously demonstrate the superiority of FastEGNN in learning to simulate physical dynamics and especially scaling to large and complicated systems.

\textbf{2.} With the help of edge dropping, FastEGNN is able to conduct inference in a substantially faster manner. Such effect is evident in the Relative Time metric, where FastEGNN with an edge dropping rate of 0.75 or 1.00 increase the inference speed. Notably, FastEGNN with 10 virtual nodes and 75\% of the edges dropped delivers the lowest MSE on $N$-body system while only using 53\% of the inference time of EGNN.

\textbf{3.} When all edges are preserved, FastEGNN performs favorably across all datasets compared with EGNN, thanks to the proposed virtual node learning technique that boosts the model expressivity. Interestingly, we also observe that the performance is only slightly affected even when a large proportion of the edges are dropped. Even in the extreme case where all edges are removed, FastEGNN still achieves competitive results. For example, on the Water-3D dataset, FastEGNN-$\langle 3, 1.00\rangle$ with all edges dropped still obtains an MSE of 3.60, remarkably lower than the MSE of 6.00 produced by EGNN. By contrast, EGNN$^\ast$ performs poorly in all cases, with performance close to Linear Dynamics. Overall, the results show that our virtual learning enhances the expressivity of the model and improves the performance, while also enabling edge dropping for higher inference speed with minimal sacrifice in prediction accuracy.

\subsubsection{Rollout Experiment}

\begin{figure*}[t!]
    \centering
    \includegraphics[width=\linewidth]{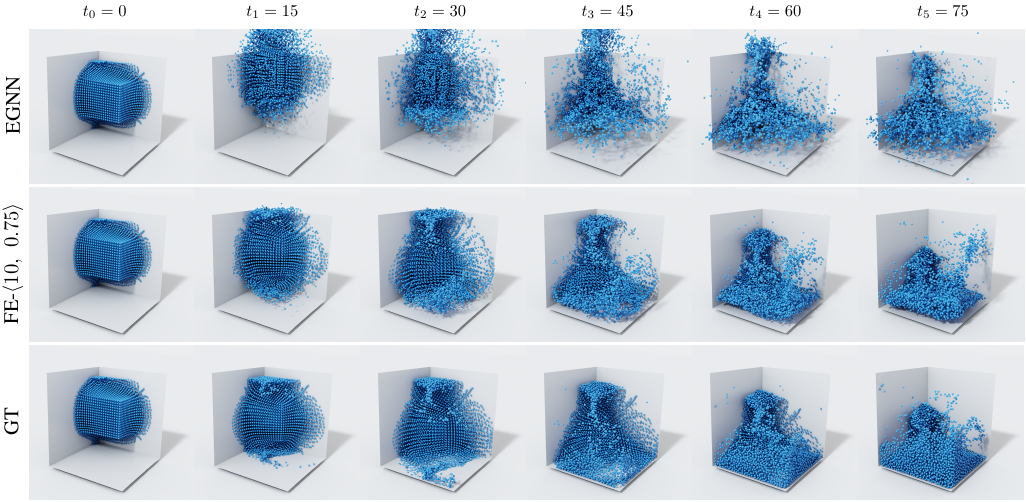}
    \vspace{-18pt}
    \caption{Rollout results of EGNN and FastEGNN-$\langle 10 ,\ 0.75\rangle$ (abbreviated as FE-$\langle 10 ,\ 0.75\rangle$ for brevity) on Water-3D dataset. After the fluid hits the box, EGNN exhibits unstable predictions, whereas our model produces results that match the ground truth more closely.}
    \label{fig:Water-3D-Rollout}
\end{figure*}

To further evaluate the robustness and accuracy of our model, we design a rollout experiment. In this setup, the model recursively uses its own predictions as inputs for the subsequent step. Compared to single-step prediction, this setting is more challenging, as prediction errors can accumulate and intensify throughout the rollout process.

We evaluate the rollout performance of EGNN and FastEGNN-$\langle10,\ 0.75\rangle$ by performing five consecutive rollout steps for each model. The time span between rollout steps is $\Delta t = 15$ frames(corresponding to 0.075 seconds), which is consistent with the configuration in our main experiments. We render the fluid shape at each step, as shown in~\cref{fig:Water-3D-Rollout}. From the visualizations, we can observe that EGNN’s predictions are highly unstable: the predicted particles are spatially dispersed, with many of them escaping far beyond the boundaries of the container. In contrast, FastEGNN yields much more stable and accurate results, with the predicted fluid shapes closely matching the ground truth. This indicates that FastEGNN can better capture the inherent physical laws.

\subsubsection{Ablation Studies}
\label{sec:exp_abl}

\begin{figure*}[t!]
    \centering
    \includegraphics[width=\linewidth]{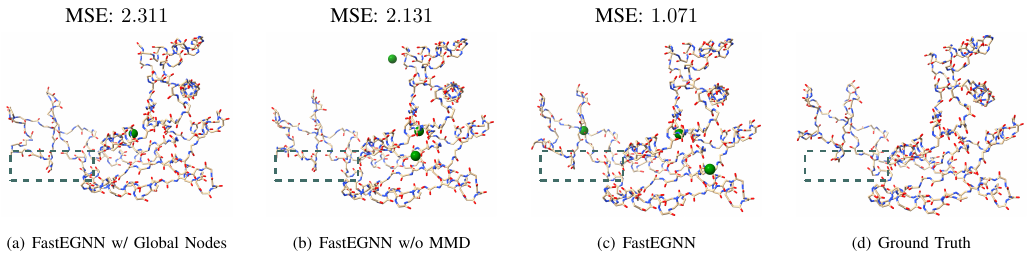}
    \vspace{-18pt}
    \caption{The results of different variants of FastEGNN on Protein Dynamics. We use green dots to indicate the learned locations of the virtual nodes. The dashed box marks highlight the region of interest where FastEGNN yields more accurate prediction. For FastEGNN w/ Global Nodes, the learned coordinates of 3 virtual nodes are almost the same.}
    \label{fig:protein}
\end{figure*}

\textbf{The number of virtual nodes.} We investigate the effect of the number of virtual nodes together with the edge dropping rate on three datasets. We sweep the number of virtual nodes in $\{1, 3, 10\}$ and the edge dropping rate in $\{0.00, 0.75, 1.00\}$, with the detailed results shown in~\cref{tab:main_result}. We also involve the results of EGNN with the same edge dropping rates for reference. It is clear that, under the same edge dropping rate, FastEGNN consistently outperforms EGNN in terms of prediction accuracy, benefiting from the incorporation of virtual nodes. In addition, it introduces only a marginal increase in inference time, owing to our efficient design of message computation and aggregation scheme for virtual nodes.

Notably, we observe that increasing the number of virtual nodes does not always lead to better performance. On the Protein Dynamics and Water-3D datasets, the best accuracy is achieved when using three virtual nodes, while the performance slightly degrades when the number increases to ten. This suggests that a limited number of virtual nodes is sufficient to enhance model performance, whereas an excessive number may introduce redundancy and reduce effectiveness.

\begin{table}[!t]
  \centering
  \caption{Ablation studies on the design of virtual node learning. Experiments are conducted on the Protein Dynamics dataset, and all variants of FastEGNN contain 3 virtual nodes.}
  \adjustbox{width=0.48\textwidth, center, padding = 2pt}{
      \begin{tabular}
      {
          p{0.2\textwidth}<{\raggedright}
          ccc
      }

        \toprule
        \multirow{3}{*}{Model Setting}   & \multicolumn{3}{c}{Dropping Rate}\\\cmidrule{2-4}
                                 & 0.00             & 0.75               & 1.00             \\
        \midrule
        EGNN                     & $2.250$          & $2.244$            & $2.257$          \\
        \midrule
        FastEGNN w/ Global Nodes & $1.950$          & $1.967$            & $2.121$          \\
        FastEGNN w/o MMD         & $1.857$          & $1.926$            & $1.975$          \\
        FastEGNN                 & $\bm{1.821}$     & $\bm{1.852}$       & $\bm{1.919}$     \\
        \bottomrule

    \end{tabular}
  }
  \label{tab:ablation}
\end{table}

\textbf{Modeling virtual nodes as an ordered set.} We have proposed to model the virtual nodes as an ordered set instead of an unordered set, which enforces mutual distinctiveness as an inductive bias. We demonstrate the validity of such a design in~\cref{tab:ablation}, where we implement a variant of FastEGNN referred to as \emph{FastEGNN w/ Global Nodes}, which treats virtual nodes as global nodes with permutation-equivariant message passing in between. 
We discover that naively viewing them as global nodes leads to worse performance and indistinguishable assignment of the virtual nodes (\cref{fig:protein}(a)). 
Their MSE is consistently higher than that of the original FastEGNN. Moreover, the performance drops sharply as the dropping rate increases from 0.00 to 1.00, with MSE rising from 1.950 to 2.121.
In contrast, our approach endows virtual nodes with orders and different roles, resulting in generally lower MSE.

\textbf{The impact of MMD loss.} As described in~\cref{sec:loss}, we introduce the MMD loss to enhance the mutual distinctiveness and global distributedness of virtual nodes, encouraging them to perform diverse and complementary roles. To assess its impact, we implement a variant without the MMD regularization, referred to as \emph{FastEGNN w/o MMD}. As shown in~\cref{tab:ablation}, removing the MMD loss during training results in a clear drop in performance. On the Protein Dynamic dataset, the model exhibits an average performance decline of an average drop of 3\% across all three edge dropping rates. This limitation is also visually evident in~\cref{fig:protein}. As shown in panel (b), without the MMD loss, all three virtual nodes are located on the right-hand side of the protein structure, failing to capture its overall geometry. In contrast, with MMD applied, FastEGNN(\cref{fig:protein}(c)) learns a more geometrically meaningful distribution of virtual nodes that better reflects the spatial characteristics of the underlying data.

\begin{table*}[ht!]
\centering
\ifrevisiontwo\color{red}\fi
\caption{\revii{Performance of EGNN~\citep{satorras2021en} and DistEGNN on the Water-3D dataset, evaluated on 1-8 devices using METIS partition strategy.}}
\label{tab:distribute_water3d_metis}
\adjustbox{width=\textwidth, center, padding = 2pt}{
\begin{tabular}{
    p{0.19\textwidth}<{\raggedright}
    p{0.115\textwidth}<{\raggedright}
    p{0.075\textwidth}<{\centering}
    p{0.075\textwidth}<{\centering}
    p{0.075\textwidth}<{\centering}
    p{0.075\textwidth}<{\centering}
    p{0.075\textwidth}<{\centering}
    p{0.075\textwidth}<{\centering}
    p{0.075\textwidth}<{\centering}
    p{0.075\textwidth}<{\centering}
}
\toprule
                                                  &                             &  $1$-device & $2$-devices & $3$-devices & $4$-devices & $5$-devices & $6$-devices & $7$-devices & $8$-devices   \\
\midrule
    \multirow{3}{*}{Graph information}            & Cutoff radius               &  $0.035$    & $0.035$     & $0.035$     & $0.035$     & $0.035$     & $0.035$     & $0.035$     & $0.035$     \\
                                                  & Edge count                  &  $94931$    & $46568$     & $30653$     & $22778$     & $18046$     & $14920$     & $12693$     & $11031$     \\
                                                  & Average degree              &  $12.16$    & $5.97$      & $3.93$      & $2.92$      & $2.31$      & $1.91$      & $1.63$      & $1.41$      \\
\midrule
    \multirow{4}{*}{MSE ($\times 10^{-4}$)}       & EGNN~\citep{satorras2021en} &  $6.00$     & $6.56$      & $6.55$      & $6.48$      & $6.84$      & $6.46$      & $6.88$      & $6.79$      \\
                                                  & DistEGNN-$1$                &  $3.25$     & $3.16$      & $3.13$      & $3.15$      & $3.75$      & $3.28$      & $3.33$      & $3.26$      \\
                                                  & DistEGNN-$3$                &  $2.58$     & $2.75$      & $2.75$      & $2.72$      & $3.07$      & $2.95$      & $2.90$      & $2.80$      \\
                                                  & DistEGNN-$10$               &  $2.66$     & $2.78$      & $2.81$      & $2.73$      & $3.09$      & $2.83$      & $2.94$      & $2.90$      \\
\midrule 
    \multirow{4}{*}{Relative inference time}      & EGNN~\citep{satorras2021en} &  $1.00$     & $0.50$      & $0.36$      & $0.31$      & $0.28$      & $0.27$      & $0.25$      & $0.24$      \\
                                                  & DistEGNN-$1$                &  $1.35$     & $0.88$      & $0.70$      & $0.63$      & $0.59$      & $0.58$      & $0.56$      & $0.56$      \\
                                                  & DistEGNN-$3$                &  $1.51$     & $0.89$      & $0.71$      & $0.64$      & $0.60$      & $0.58$      & $0.56$      & $0.56$      \\
                                                  & DistEGNN-$10$               &  $1.81$     & $1.09$      & $0.87$      & $0.75$      & $0.69$      & $0.67$      & $0.63$      & $0.63$      \\
\midrule
    \multirow{4}{*}{Device memory usage (GB)}     & EGNN~\citep{satorras2021en} &  $74.15$    & $34.47$     & $22.44$     & $15.60$     & $10.21$     & $9.79$      & $8.09$      & $5.99$      \\
                                                  & DistEGNN-$1$                &  $75.66$    & $35.17$     & $22.93$     & $15.92$     & $10.46$     & $10.02$     & $8.26$      & $6.17$      \\
                                                  & DistEGNN-$3$                &  $78.09$    & $36.44$     & $23.76$     & $16.53$     & $10.94$     & $10.43$     & $8.60$      & $6.49$      \\
                                                  & DistEGNN-$10$               &  $84.26$    & $39.54$     & $25.46$     & $17.68$     & $13.61$     & $11.79$     & $10.58$     & $6.49$      \\
\bottomrule
\end{tabular}
}
\end{table*}

\begin{table*}[ht!]
\centering
\ifrevisiontwo\color{red}\fi
\caption{\revii{Performance of EGNN~\citep{satorras2021en} and DistEGNN on the Fluid113K dataset, evaluated on 1-8 devices using METIS partition strategy.}}
\label{tab:distribute_largefluid_metis}
\adjustbox{width=\textwidth, center, padding = 2pt}{
\begin{tabular}{
    p{0.19\textwidth}<{\raggedright}
    p{0.115\textwidth}<{\raggedright}
    p{0.075\textwidth}<{\centering}
    p{0.075\textwidth}<{\centering}
    p{0.075\textwidth}<{\centering}
    p{0.075\textwidth}<{\centering}
    p{0.075\textwidth}<{\centering}
    p{0.075\textwidth}<{\centering}
    p{0.075\textwidth}<{\centering}
    p{0.075\textwidth}<{\centering}
}
\toprule
                                                  &                             &  $1$-device & $2$-devices & $3$-devices & $4$-devices & $5$-devices & $6$-devices & $7$-devices & $8$-devices  \\
\midrule
    \multirow{3}{*}{Graph information}            & Cutoff radius               &  $0.075$    & $0.075$     & $0.075$     & $0.075$     & $0.075$     & $0.075$     & $0.075$     & $0.075$      \\
                                                  & Edge count                  &  $1706973$  & $843062$    & $558533$    & $417052$    & $332228$    & $275649$    & $235694$    & $205878$     \\
                                                  & Average degree              &  $15.09$    & $7.45$      & $4.94$      & $3.69$      & $2.94$      & $2.44$      & $2.08$      & $1.82$       \\
\midrule
    \multirow{5}{*}{MSE ($\times 10^{-2}$)}       & EGNN~\citep{satorras2021en} &  $17.15$    & $18.07$     & $18.65$     & $18.63$     & $18.36$     & $19.36$     & $18.43$     & $18.71$      \\
                                                  & DistEGNN-$1$                &  $8.93$     & $9.30$      & $9.22$      & $9.71$      & $9.57$      & $9.52$      & $9.39$      & $9.41$       \\
                                                  & DistEGNN-$3$                &  $8.67$     & $8.76$      & $8.53$      & $8.62$      & $8.64$      & $8.75$      & $8.88$      & $8.54$       \\
                                                  & DistEGNN-$10$               &  $7.99$     & $8.30$      & $7.99$      & $8.65$      & $8.83$      & $8.59$      & $8.50$      & $8.67$       \\
                                                  & DistEGNN-$20$               &  $-$        & $8.78$      & $8.98$      & $8.99$      & $8.87$      & $8.76$      & $8.56$      & $8.77$       \\
\midrule 
    \multirow{5}{*}{Relative inference time}      & EGNN~\citep{satorras2021en} & $1.00$      & $0.53$      & $0.38$      & $0.34$      & $0.30$      & $0.28$      & $0.26$      & $0.25$       \\
                                                  & DistEGNN-$1$                & $1.59$      & $0.80$      & $0.64$      & $0.55$      & $0.51$      & $0.49$      & $0.48$      & $0.47$       \\
                                                  & DistEGNN-$3$                & $1.66$      & $0.88$      & $0.66$      & $0.57$      & $0.51$      & $0.49$      & $0.48$      & $0.47$       \\
                                                  & DistEGNN-$10$               & $2.00$      & $1.04$      & $0.77$      & $0.67$      & $0.58$      & $0.54$      & $0.51$      & $0.48$       \\
                                                  & DistEGNN-$20$               & $2.45$      & $1.32$      & $0.97$      & $0.79$      & $0.70$      & $0.64$      & $0.60$      & $0.57$       \\
\midrule
    \multirow{5}{*}{Device memory usage (GB)}     & EGNN~\citep{satorras2021en} & $41.55$     & $21.57$     & $18.35$     & $12.54$     & $10.69$     & $9.51$      & $7.69$      & $5.73$       \\
                                                  & DistEGNN-$1$                & $46.57$     & $21.70$     & $18.87$     & $12.76$     & $10.96$     & $9.69$      & $7.86$      & $5.89$       \\
                                                  & DistEGNN-$3$                & $48.12$     & $21.95$     & $20.06$     & $12.95$     & $11.51$     & $10.09$     & $8.22$      & $6.18$       \\
                                                  & DistEGNN-$10$               & $55.88$     & $22.71$     & $20.95$     & $13.18$     & $12.46$     & $10.12$     & $8.36$      & $6.27$       \\
                                                  & DistEGNN-$20$               & $73.25$     & $38.74$     & $22.00$     & $15.44$     & $14.01$     & $11.08$     & $8.47$      & $7.81$       \\
\bottomrule
\end{tabular}
}
\end{table*}

\subsection{Evaluation of DistEGNN}
\label{sec:exp_scalability}

In~\cref{sec:distributed_fastegnn}, we present a distributed approach for geometric GNNs based on graph partition. This method significantly enhances the model's scalability, making it feasible to simulate geometric graphs with millions of edges. In this section, we design a set of challenging experiments to validate the advantages of DistEGNN in terms of computational overhead and prediction accuracy. The experiment setup and corresponding results are detailed below.

\subsubsection{Implementation Details}

We evaluate the performance of EGNN and DistEGNN on two datasets: Water-3D and Fluid113K. Fluid113K is our generated extremely large-scale dataset, with each graph containing over 100K nodes and more than 1M edges. Simulating such large-scale graphs requires over 40GB of memory even with a batch size of 1, making it impossible to run on low-memory devices such as NVIDIA V100-32G. However, with our distributed strategy, the memory burden on each device is greatly reduced, enabling multiple low-memory devices to jointly simulate one large graph.

For the Water-3D dataset, we set the batch size to 16 and the virtual nodes to $\{1, 3, 10\}$. For Fluid113K, we use a batch size of 1 and the virtual nodes in \revii{$\{1, 3, 10, 20\}$}. Each model is trained using 1 to 8 computing devices. To measure relative inference time and peak memory usage per device, we perform training and inference 14 times in a consistent environment, discard the two highest and two lowest values, and report the average of the remaining runs as the final result.

\revii{The choice of graph partition strategy significantly affects model performance. In our main experiments, we adopt METIS~\cite{metis}, an efficient graph partition algorithm that minimizes cut edges between subgraphs while maintaining balanced partition sizes. This not only reduces information loss caused by partitioning but also ensures balanced computational workloads across devices. An analysis of alternative partition strategies on DistEGNN’s performance is provided in~\cref{sec:dist_ablation}. We provide the detailed graph construction and training hyperparameters in~\cref{appendix:experiment_details}.}

\subsubsection{Main Results}

\begin{figure*}[ht!]
    \centering
    \includegraphics[width=\linewidth]{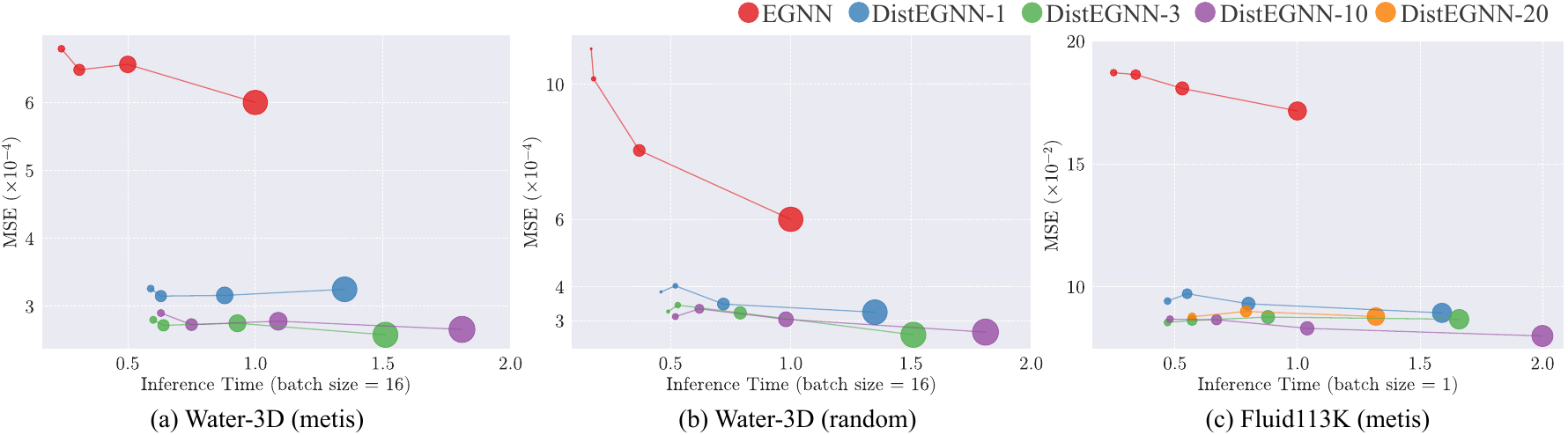}
    \vspace{-18pt}
    \caption{\revii{Visualization of inference time, accuracy, and memory usage of EGNN~\citep{satorras2021en} and DistEGNN on the Water-3D and Fluid113K datasets under metis and random graph partition strategies. Larger circles indicate higher memory usage. For each model, we report the results on 1, 2, 4, and 8 devices.}}
    \label{fig:distribute_main}
\end{figure*}

Our main experimental results are presented in~\cref{tab:distribute_water3d_metis,tab:distribute_largefluid_metis}. For better illustration, we visualize the inference time, accuracy, and peak memory usage of different models across varying numbers of devices in~\cref{fig:distribute_main}. \revii{We also present a visualization of models' performance on Water-3D dataset under the random partition strategy in~\cref{fig:distribute_main} (b), with the corresponding numerical results reported in~\cref{tab:distribute_water3d_random_same_radius}.}
\rev{In addition, we provide the training curves of FastEGNN and DistEGNN in~\cref{appendix:training_curves}, demonstrating that the subgraph partition and distributed computation have no adverse impact on training stability or model expressiveness.}
The analysis of the results leads to the following conclusions:

\textbf{1.} Our enhanced model consistently demonstrates better prediction accuracy than EGNN. As the~\cref{fig:distribute_main} shows, the worst result of the DistEGNN still surpasses the best performance achieved by EGNN. For instance, \revii{DistEGNN-1 achieves an MSE of 3.26} on the Water-3D dataset even when running in parallel on 8 computing devices, which is substantially lower than EGNN's best result of 6.00 under a single-device setup. In addition, its inference time is reduced to just \revii{56\%} of that required by the single-device EGNN.

\textbf{2.} Our distributed parallelization approach significantly reduces both computational time and memory consumption, making it feasible to simulate extremely large graphs using multiple less powerful devices with limited memory. As shown in~\cref{tab:distribute_largefluid_metis}, processing a single graph of Fluid113K on a single device requires 42GB of memory for EGNN. These memory requirements exceed the capabilities of many common devices. With our distributed strategy applied, both time and memory costs are significantly reduced in practice. \revii{For example, the inference time of DistEGNN-3 with two cooperating devices on this dataset is reduced by 47\%, resulting in a $1.89 \times$ speedup. The memory usage per device also drops from about 48GB to 22GB, representing a 54\% reduction. When using 8 devices, inference is $3.5 \times$ faster, and the memory required per device is reduced to 6.18GB, enabling simulation of graphs with over 100,000 nodes with minimal resources.}

\textbf{3.} Our virtual nodes effectively mitigate performance degradation caused by graph partition. As shown in~\cref{fig:distribute_main}, when the number of computing devices increases from 1 to 8, EGNN exhibits a much more severe performance drop compared to DistEGNN. \revii{Quantitatively, 
on the Water-3D dataset, EGNN's performance declines by 13\% under the 8 devices setting, while DistEGNN-3 incurs only an 8\% loss. Under random partition, the advantages of DistEGNN become more pronounced (\cref{fig:distribute_main} (b)): EGNN experiences an 84\% drop in accuracy, whereas DistEGNN-10 shows only a minor reduction of 17\%.} We attribute this advantage to the global information transmitted by virtual nodes across devices, which effectively bridges the information gap between subgraphs caused by distributed graph partition. This allows DistEGNN to maintain high accuracy under distributed settings, enhancing both scalability and practical applicability.

\begin{figure*}[t!]
    \centering
    \includegraphics[width=\linewidth]{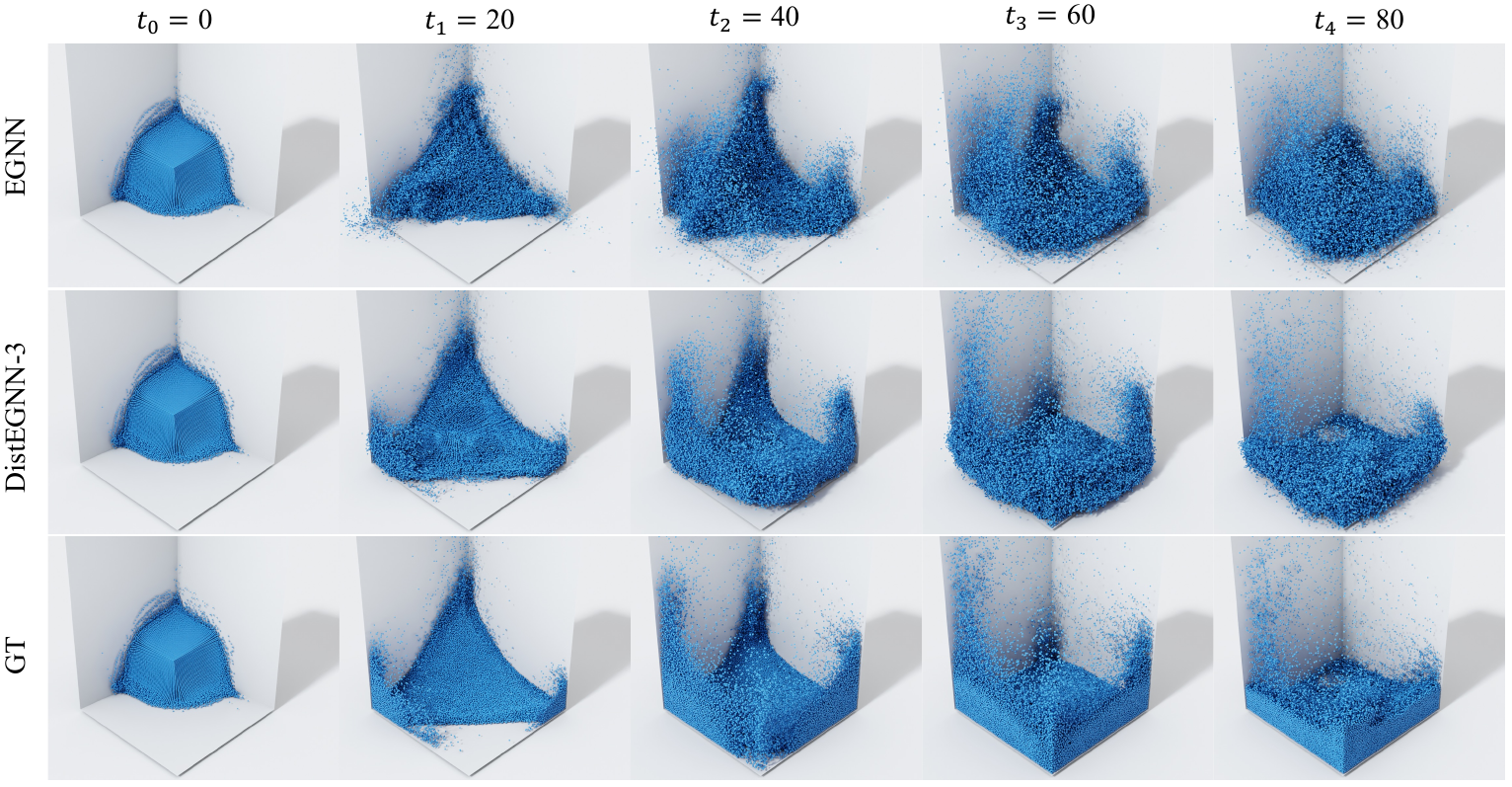}
    \vspace{-18pt}
    \caption{\revii{Rollout results of EGNN and DistEGNN-3 on Fluid113K dataset, with training and inference performed in the distributed environment with 8 devices.}}
    \label{fig:Fluid113K-Rollout}
\end{figure*}

\subsubsection{Rollout Experiment}

We further evaluate the performance of DistEGNN and EGNN on rollout experiments under the distributed setting. The rollout experiment involves recursively using the current prediction as the input for the next step, which accumulates errors and thus places higher demands on the model’s predictive accuracy. 
We conduct this experiment on the test set of Fluid113K dataset with 8 devices to collaboratively simulate a large graph. Results are visualized in~\cref{fig:Fluid113K-Rollout}. 
As shown in the figure, the predictions of DistEGNN-3 align more closely with the ground truth compared to those of EGNN.
\revii{This is particularly evident near the lower boundary region of the container.
DistEGNN achieves higher accuracy when predicting water colliding with the container walls, effectively capturing the constraint imposed by the rigid boundary. By comparison, EGNN shows substantial particle dispersion outside the container, resulting in disordered predictions that do not faithfully capture the collision and rebound dynamics.}

\subsubsection{Ablation Studies}
\label{sec:dist_ablation}

\rev{
\textbf{The number of virtual nodes} This ablation study explores two aspects: (1) the performance trend with varied numbers of virtual nodes on the same dataset, and (2) the optimal virtual node configuration across different datasets.
}

\rev{
In the main experiment, we benchmark the performance of DistEGNN with 1, 3, and 10 virtual nodes on the Water-3D dataset. The results are presented in~\cref{tab:distribute_water3d_metis}. We observe that the model achieves progressively higher prediction accuracy as the number of virtual nodes increases from 1 to 3. However, when the number increases to 10, the model's performance becomes consistently worse than that with 3 virtual nodes, except only when using 6 devices. These results suggest that an appropriate number of virtual nodes facilitates cross-cluster information flow and improves model expressiveness. Yet, when the virtual nodes become excessive, redundant nodes may introduce noise and unwanted interference, weakening representation quality.}

\begin{figure*}[t!]
    \centering
    \ifrevision\color{red}\fi
    \includegraphics[width=\linewidth]{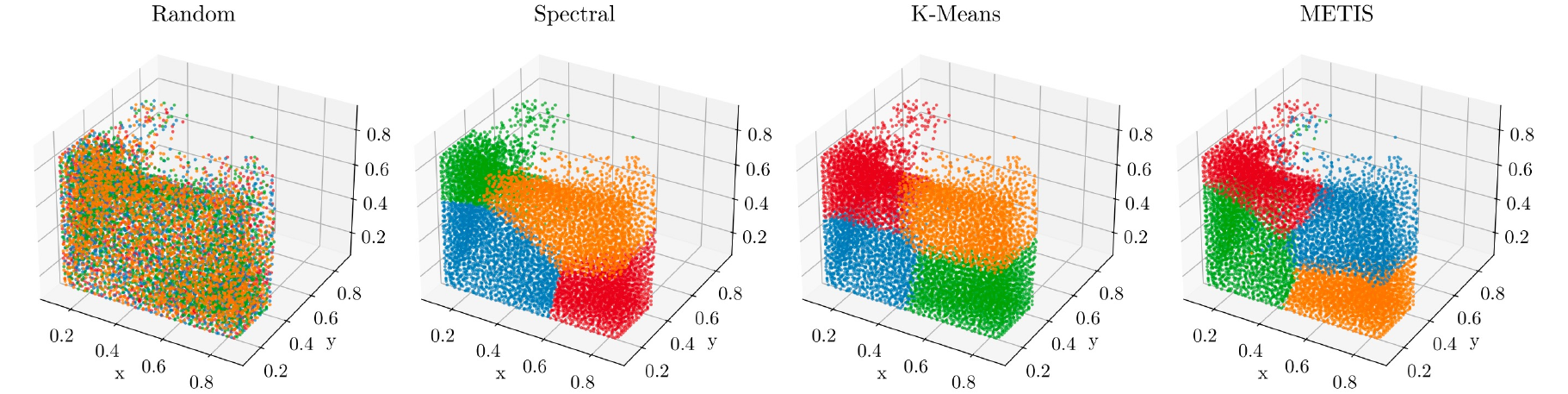}
    \vspace{-18pt}
    \caption{Visualization of node distribution across four devices under random, spectral, K-Means, and METIS partition strategies.}
    \label{fig:graph_partition}
\end{figure*}

\rev{
In~\cref{tab:distribute_largefluid_metis}, we evaluate the performance of DistEGNN on the Fluid113K dataset with \revii{1, 3, 10, 20} virtual nodes. In contrast to the Water-3D results, where DistEGNN-3 performed best, DistEGNN-10 achieves superior performance in most cases compared to models with 1 or 3 virtual nodes. This suggests that the optimal number of virtual nodes is data-dependent, with larger graphs potentially requiring more nodes to model global structures. However, further increasing the count to 20 (DistEGNN-20) leads to performance degradation, indicating a point of diminishing returns. We conclude that when simulating large-scale geometric graphs with DistEGNN, it is crucial to balance the number of virtual nodes to mitigate message loss caused by distributed graph partition, which depends on factors such as graph scale, partition strategy, and the number of devices.
}

\begin{table*}[t!]
\centering
\ifrevision\color{red}\fi
\caption{Performance of DistEGNN with 3 virtual nodes under different partitioning methods on Water-3D dataset. For all metrics, lower values indicate better performance. We rank the four strategies for each metric, and report the average rank in the last column.}
\label{tab:graph_partition}
\adjustbox{width=\textwidth, center}{
\begin{tabular}{
    p{0.185\textwidth}<{\raggedright}
    p{0.08\textwidth}<{\raggedright}
    p{0.072\textwidth}<{\centering}
    p{0.072\textwidth}<{\centering}
    p{0.072\textwidth}<{\centering}
    p{0.072\textwidth}<{\centering}
    p{0.072\textwidth}<{\centering}
    p{0.072\textwidth}<{\centering}
    p{0.072\textwidth}<{\centering}
    p{0.072\textwidth}<{\centering}
}
\toprule
                                                  &                             & $2$-devices & $3$-devices & $4$-devices & $5$-devices & $6$-devices & $7$-devices & $8$-devices  & Avg. Rank \\
\midrule
    \multirow{4}{*}{Edge Count}                   & Random                      & $23711$     & $10548$     & $5930$      & $3794$      & $2637$      & $1936$      & $1481$       & $1$ \\
                                                  & Spectral                    & $46524$     & $30683$     & $22837$     & $18115$     & $14988$     & $12711$     & $11048$      & $3$ \\
                                                  & K-Means                     & $46560$     & $30680$     & $22812$     & $18085$     & $14964$     & $12729$     & $11082$      & $4$ \\
                                                  & METIS                       & $46568$     & $30653$     & $22778$     & $18046$     & $14920$     & $12693$     & $11031$      & $2$ \\
\midrule
    \multirow{4}{*}{MSE ($\times 10^{-4}$)}       & Random                      & $3.23$      & $3.26$      & $3.46$      & $3.47$      & $3.19$      & $3.45$      & $3.27$       & $4$ \\
                                                  & Spectral                    & $3.17$      & $3.10$      & $3.15$      & $2.96$      & $2.80$      & $3.20$      & $3.00$       & $3$ \\
                                                  & K-Means                     & $2.73$      & $2.67$      & $3.29$      & $2.92$      & $2.98$      & $2.74$      & $2.98$       & $2$ \\
                                                  & METIS                       & $2.75$      & $2.75$      & $2.72$      & $3.07$      & $2.95$      & $2.90$      & $2.80$       & $1$ \\
\midrule
    \multirow{4}{*}{Relative inference time}      & Random                      & $0.79$      & $0.60$      & $0.53$      & $0.49$      & $0.48$      & $0.49$      & $0.49$       & $1$ \\
                                                  & Spectral                    & $0.93$      & $0.74$      & $0.68$      & $0.65$      & $0.60$      & $0.60$      & $0.60$       & $4$ \\
                                                  & K-Means                     & $0.94$      & $0.74$      & $0.66$      & $0.60$      & $0.59$      & $0.57$      & $0.57$       & $3$ \\
                                                  & METIS                       & $0.89$      & $0.71$      & $0.64$      & $0.60$      & $0.58$      & $0.56$      & $0.56$       & $2$ \\
\midrule
    \multirow{4}{*}{Device memory usage (GB)}     & Random                      & $19.32$     & $6.42$      & $4.31$      & $3.36$      & $2.45$      & $1.70$      & $1.40$       & $1$ \\
                                                  & Spectral                    & $31.97$     & $22.26$     & $19.00$     & $13.12$     & $12.88$     & $10.08$     & $9.44$       & $4$ \\
                                                  & K-Means                     & $29.09$     & $22.74$     & $17.10$     & $14.22$     & $10.31$     & $8.05$      & $7.48$       & $2$ \\
                                                  & METIS                       & $36.44$     & $23.76$     & $16.53$     & $10.94$     & $10.43$     & $8.60$      & $6.49$       & $3$ \\    
\bottomrule
\end{tabular}
}

\end{table*}

\revii{
\textbf{Graph partition strategy.} The choice of graph partition strategy is an important hyperparameter. In our main experiments, we adopt the METIS partition strategy. In this section, we explore more strategies, including the random method, the spectral method, and K-Means. Among them, K-Means clusters at the node level by minimizing within-cluster distances, whereas spectral performs graph-level partition and seeks to minimize the number of cross-partition edges. The implementation detail can be found in~\cref{appendix:graph_partition}. \cref{fig:graph_partition} shows the node distribution across four devices obtained from the four strategies on the same input graph. The random partition distributes the nodes almost uniformly across the four devices, whereas other methods yield much more localized and compact results.
}

\rev{
We conduct experiments on the Water-3D dataset with 2-8 distributed devices, and present the results in~\Cref{tab:graph_partition}.
The structure-aware partition strategies produce partitions with enhanced locality and thus retain more edges than the naive random partition under the same cutoff radius. Accordingly, the inference time and memory consumption are also higher, while their accuracy is consistently better than that of naive random partition. We attribute this to the fact that structure‐aware graph partition strategies remove fewer cut edges between subgraphs, thus incurring less information loss. Such loss can be more easily compensated through virtual-node message passing, leading to a smaller performance degradation in distributed settings.
}

\rev{
Meanwhile, we observe that although the three structure-aware partition strategies preserve a nearly identical number of edges, their performance outcomes diverge. Among them, METIS enables DistEGNN to achieve the best average rank in MSE across 2-8 devices. For example, with four devices, METIS attains an MSE of 2.72, which is substantially lower than the values achieved by spectral partition (3.15) and K-Means (3.29). Furthermore, METIS proves both effective and efficient. By minimizing cross-partition edges while balancing node counts across partitions, it not only delivers strong performance but also prevents idle time caused by workload imbalance among devices. This is evidenced in~\Cref{tab:graph_partition}, where METIS achieves higher inference speeds than spectral or K-Means partition. These results confirm its strong suitability for DistEGNN's application scenarios.
}

\section{Conclusion}

We propose FastEGNN and DistEGNN, two advanced models that leverage virtual nodes to operate efficiently and effectively on large sparse geometric graphs.
The core insight of FastEGNN lies in constructing an ordered set of virtual nodes that enjoys both mutual distinctiveness and global distributedness to perform expressive message passing. 
DistEGNN, the distributed version of FastEGNN, enables multiple devices collaboratively process an extremely large graph, while simultaneously leverage a shared set of virtual nodes to transfer global messages across devices. This further extends the model’s capability and scalability.
Comprehensive evaluations on both small graph with 100 nodes and large-scale graph with more than 100K nodes consistently demonstrate the superiority of FastEGNN and DistEGNN in terms of achieving remarkably lower simulation error and offering substantial efficiency improvements due to sparsification and distribution.
\par

\section*{Acknowledgements}
This work was jointly supported by the following projects: the National Natural Science Foundation of China (No. 62376276); Beijing Nova Program (No. 20230484278); the Fundamental Research Funds for the Central Universities, and the Research Funds of Renmin University of China (23XNKJ19); Public Computing Cloud, Renmin University of China. 

{
\small
\bibliographystyle{IEEEtran}
\bibliography{Reference}
}

\onecolumn
\appendix
\setcounter{page}{1}
\subsection{Proof}\label{appendix:proof}
\begin{theorem}[\cref{prop:equofnn}]\label{th:equofnn}
If the initialization of the virtual nodes satisfies~\cref{eq:eqconstraint-vn}, then after~\cref{eq:mij,eq:mv,eq:miv,eq:xr(l+1),eq:hr(l+1),eq:Xv(l+1),eq:Hv(l+1)}, the output coordinates $\vec\mX^{(L)}$ are E(3)-equivariant and permutation-equivaraint, the virtual coordinates $\vec\mZ^{(L)}$ are E(3)-equivariant and permutation-invariant, with respect to the input $\vec\mX^{(0)}$.
\end{theorem}
\begin{proof}
Consider a sequence composed of functions $\{\phi_i:\gX^{(i-1)}\rightarrow\gX^{(i)}\}_{i=1}^N$ equivariant to a same group $G$, the equivariance lead to an interesting property that  
\begin{equation*}
    \phi_N\circ\cdots\circ\phi_{i+1}\circ \rho_{\gX^{(i)}}(g)\phi_{i}\circ\cdots\circ\phi_1=\phi_N\circ\cdots\circ\phi_{j+1}\circ \rho_{\gX^{(j)}}(g)\phi_{j}\circ\cdots\circ\phi_1,
\end{equation*}
holds for all $i,j=1,2,\dots,N$ and $g\in G$, which means that the group elements $g$ can be freely exchanged in the composite sequence of equivariant functions. In particular, if one of the equivariant functions (\emph{e.g.} $\phi_k$) is replaced by an invariant function, the group element $g$ will be absorbed, that means 
\begin{equation*}
    \phi_N\circ\cdots\circ\phi_k\circ\cdots\circ\phi_{i+1}\circ \rho_{\gX^{(i)}}(g)\phi_{i}\circ\cdots\circ\phi_1=\phi_N\circ\cdots\circ\phi_1.
\end{equation*}
holds for all $g\in G$ but only $i=1,2,\dots,k$. Although $\phi_N\circ\cdots\circ\phi_k$ is still equivariant, because the group elements must be input starting from $\phi_1$, the overall $\phi_N\circ\cdots\circ\phi_1$ is still an invariant function.

Since \cref{eq:eqconstraint-vn} is the definition of a function with $\mathrm{E}(3)$-equivariance and permutation-invaraince, and the pooling operators like \texttt{sum} or \texttt{mean} maintain permutation-equivariance of the whole model. Now to prove the $\mathrm{E}(3)$-equivarianceof the whole model, we choose to give a stronger conclusion, namely each of \cref{eq:mij,eq:mv,eq:miv,eq:hr(l+1),eq:Hv(l+1)} is an $\mathrm{E}(3)$-invariant function, and each of \cref{eq:xr(l+1),eq:Xv(l+1)} is an $\mathrm{E}(3)$-equivariant function.

It is obvious that each input in \cref{eq:mij,eq:mv,eq:miv,eq:hr(l+1),eq:Hv(l+1)} is either a constant or an inner product term, thus ensuring the $\mathrm{E}(3)$-invariance of each function. And each of \cref{eq:xr(l+1),eq:Xv(l+1)} is linear combination of three-dimensional vectors with an $\mathrm{E}(3)$-invariant weight, thus ensuring the $\mathrm{E}(3)$-equivariance of each function.
\end{proof}

\begin{theorem}\label{th:transformation}
The $\mathrm{E}(3)$-equivaraint function $f(\vec\vx_i, \vec\mZ)$ in \cref{prop:xZ} can be decomposed into an $\mathrm{O}(3)$-equivaraint and translation-invariant function of $\vec\mZ-\vec\vx_i$, and the addition of $\vec\vx_i$.
\end{theorem}
\begin{proof}
Since $f(\vec\vx_i, \vec\mZ)$ is equivaraint to translation $\vec\vt$, considering $\vec\vt=-\vec\vx_i$, we get
\begin{equation*}
    f(\vec\vx_i, \vec\mZ)=f(\vec\vx_i, \vec\mZ)-\vec\vx_i+\vec\vx_i=f(\vec\vx_i-\vec\vx_i, \vec\mZ-\vec\vx_i)+\vec\vx_i=f(\vec{\bm{0}}, \vec\mZ-\vec\vx_i)+\vec\vx_i\coloneqq h(\vec\mZ-\vec\vx_i)+\vec\vx_i.
\end{equation*}
\end{proof}

\begin{lemma}\label{le:span}
For any O(3)-equivariant function $\hat{f}(\Vec{\mZ})$, it must fall into the subspace spanned by the columns of $\Vec{\mZ}$, namely, there exists a function $s(\Vec{\mZ})$, satisfying $\hat{f}(\Vec{\mZ})=\Vec{\mZ}s(\Vec{\mZ})$.
\end{lemma}
\begin{proof}
The proof is given by~\cite{villar2021scalars}. Essentially, suppose $\Vec{\mZ}^{\perp}$ is the orthogonal complement of the column space of $\Vec{\mZ}$. Then there must exit functions $s(\Vec{\mZ})$ and $s^{\perp}(\Vec{\mZ})$, satisfying $\hat{f}(\Vec{\mZ})=\Vec{\mZ} s(\Vec{\mZ}) + \Vec{\mZ}^{\perp}s^{\perp}(\Vec{\mZ})$. We can always find an orthogonal transformation $\mO$ allowing $\mO\Vec{\mZ}=\Vec{\mZ}$ while $\mO\Vec{\mZ}^{\perp}=-\Vec{\mZ}^{\perp}$. With this transformation $\mO$, we have $\hat{f}(\mO\Vec{\mZ})=\hat{f}(\Vec{\mZ})=\Vec{\mZ} s(\Vec{\mZ}) + \Vec{\mZ}^{\perp}s^{\perp}(\Vec{\mZ})$, and $\mO\hat{f}(\Vec{\mZ})=\Vec{\mZ} s(\Vec{\mZ}) - \Vec{\mZ}^{\perp}s^{\perp}(\Vec{\mZ})$. The equivariance property of $\hat{f}$ implies $\Vec{\mZ} s(\Vec{\mZ}) + \Vec{\mZ}^{\perp}s^{\perp}(\Vec{\mZ})=\Vec{\mZ} s(\Vec{\mZ}) - \Vec{\mZ}^{\perp}s^{\perp}(\Vec{\mZ})$, which derives $s^{\perp}(\Vec{\mZ})=0$. Hence, the proof is concluded. 
\end{proof}

\begin{lemma}\label{le:inv}
If the O(3)-equivariant function $\hat{f}(\Vec{\mZ})$ lies in the subspace spanned by the columns of $\Vec{\mZ}$, then there exists a function $\sigma$ satisfying $\hat{f}(\Vec{\mZ}) =\Vec{\mZ}\sigma(\Vec{\mZ}^\top\mZ)$.
\end{lemma}
\begin{proof}
The proof is provided by Corollary 2 in~\cite{huang2022equivariant}. The basic idea is that $\hat{f}(\Vec{\mZ})$ can be transformed to $\hat{f}(\Vec{\mZ})=\Vec{\mZ}\eta(\Vec{\mZ})$ where $\eta(\Vec{\mZ})$ is O(3)-invariant. According to Lemma 1-2 in~\cite{huang2022equivariant}, $\eta(\Vec{\mZ})$ must be written as $\eta(\Vec{\mZ})=\sigma(\Vec{\mZ}^\top\Vec{\mZ})$, which completes the proof. 
\end{proof}

\begin{theorem}\label{th:xzlocal-global}
There exists a surjection from $dao$ to $(\vec\mZ-\vec\vx_i)^\top (\vec\mZ-\vec\vx_i)$, that is
\begin{equation*}
    \exists \text{~a funcion}\ \zeta,\ (\vec\mZ-\vec\vx_i)^\top (\vec\mZ-\vec\vx_i)=\zeta\left(\bigoplus_{c=1}^C \|\vec\vz_c-\vec\vx_i\|^2, (\vec\mZ-\bar\vx)^\top(\vec\mZ-\bar\vx)\right).
\end{equation*}
\end{theorem}
\begin{proof}
First, we simplify the notation as
\begin{equation*}
    \mA\coloneqq(\vec\mZ-\vec\vx_i)^\top (\vec\mZ-\vec\vx_i),\quad \mD\coloneqq(\vec\mZ-\bar\vx)^\top(\vec\mZ-\bar\vx),\quad \mY\coloneqq\bigoplus_{c=1}^C \|\vec\vz_c-\vec\vx_i\|^2.
\end{equation*}
Since $\vec\mZ-\vec\vx_i=(\vec\mZ-\bar\vx)-(\vec\vx_i-\bar\vx)$, letting $\vec\mZ'=\vec\mZ-\bar\vx$, $\vec\vx'=\vec\vx_i-\bar\vx$, $\vec\vz'_c=\vec\mZ'_{:,c}$, we have
\begin{equation*}
    \mA=(\vec\mZ'-\vec\vx')^\top(\vec\mZ'-\vec\vx'),\quad \mD=\vec\mZ'^\top\vec\mZ',\quad \mY = \bigoplus_{c=1}^C \|\vec\vz'_c-\vec\vx'\|^2.
\end{equation*}
Thus, we find the following equations:
\begin{equation*}
\begin{aligned}
\mA_{mn}&=\left\langle
    \vec\vz'_m-\vec\vx',\ \vec\vz'_n-\vec\vx'
\right\rangle\\
&=\left\langle
    \vec\vz'_m,\ \vec\vz'_n
\right\rangle
+\frac{1}{2}\left(
    \left\langle
        \vec\vx',\ \vec\vx'-2\vec\vz'_m
    \right\rangle
    +\left\langle
        \vec\vx',\ \vec\vx'-2\vec\vz'_n
    \right\rangle
\right)\\
&=\left\langle
    \vec\vz'_m,\ \vec\vz'_n
\right\rangle
+\frac{1}{2}\left(\left(\|\vec\vz'_m-\vec\vx'\|^2-\|\vec\vz'_m\|^2\right)+\left(\|\vec\vz'_n-\vec\vx'\|^2-\|\vec\vz'_n\|^2\right)\right)\\
&=\mD_{mn}+\frac{1}{2}\left((\mY_{m}-\mD_{mm})+(\mY_{n}-\mD_{nn})\right),
\end{aligned}
\end{equation*}
and
\begin{align*}
        \mY_{m}=\mA_{mm},\ \mD_{mn}=\mA_{mn}-\frac{1}{2}(\mA_{mm}+\mA_{nn})+\frac{1}{2}(\mD_{mm}+\mD_{nn}),
\end{align*}
which means there is a a surjection from $\bigoplus_{c=1}^C \|\vec\vz_c-\vec\vx_i\|^2, (\vec\mZ-\bar\vx)^\top(\vec\mZ-\bar\vx)$ to $(\vec\mZ-\vec\vx_i)^\top (\vec\mZ-\vec\vx_i)$. 

\end{proof}

\begin{theorem}[\cref{prop:xZ}]\label{theo:universal}
The update function must take the form $f(\vec\vx_i, \vec\mZ)=\vec\vx_i+\sum_{c=1}^C (\vec\vz_c-\vec\vx_i)\psi_c\left(\bigoplus_{c=1}^C \|\vec\vz_c-\vec\vx_i\|^2, \vm^v\right)$, where $\psi_c:\R^{C+C^2}\rightarrow\R$ is an arbitrary non-linear function, and $\vm^v$ is an E(3)-invariant term by~\cref{eq:miv}.
\end{theorem}
\begin{proof}
According to \cref{th:transformation}, there exists an $\mathrm{O}(3)$-equivaraint and transform-invariant function $h$, thus
\begin{equation*}
    f(\vec\vx_i, \vec\mZ)=h(\vec\mZ-\vec\vx_i)+\vec\vx_i.
\end{equation*}
Then, from \cref{le:inv}, we rewrite the $\mathrm{O}(3)$-equivarint $h$ into a product of the vector $\vec\mZ-\vec\vx_i$ and a scalar function $\eta$ with an inner-product input, that is, 
\begin{equation*}
    f(\vec\vx_i, \vec\mZ)=(\vec\mZ-\vec\vx_i)\eta\left((\vec\mZ-\vec\vx_i)^\top(\vec\mZ-\vec\vx_i)\right)+\vec\vx_i.
\end{equation*}
Finally, with \cref{th:xzlocal-global}, we transform it into a more expressive function as
\begin{equation*}
    f(\vec\vx_i, \vec\mZ)=(\vec\mZ-\vec\vx_i)\psi\left(\bigoplus_{c=1}^C \|\vec\vz_c-\vec\vx_i\|^2, (\vec\mZ-\bar\vx)^\top(\vec\mZ-\bar\vx)\right)+\vec\vx_i,
\end{equation*}
which is equivalent to \cref{eq:miv}.
\end{proof}

\rev{
\begin{theorem}[\cref{prop:equi_distegnn_init}]
The CoM initialization of DistEGNN is E(3)-equivariant.
\end{theorem}
\begin{proof}
The CoM initialization can be written as follows:
\begin{equation*}
    \bar\vx = \frac{1}{N} \sum_{d=1}^{D} \left(\sum_{n_d=1}^{N_{d}} \gvx_{n_d}\right).
\end{equation*}
Since it involves only linear combinations, which preserve O(3)-equivariance, it suffices to show translation equivariance. Let us consider a translation $\gvt$, for which we have:
\begin{equation*}
    \frac{1}{N} \sum_{d=1}^{D} \left(\sum_{n_d=1}^{N_{d}} (\gvx_{n_d}+\gvt\ )\right)
    = \frac{1}{N} \sum_{d=1}^{D} \left(\sum_{n_d=1}^{N_{d}} \gvx_{n_d}\right)
    + \frac{1}{N} \sum_{d=1}^{D} \left(\sum_{n_d=1}^{N_{d}} \gvt\right)
    = \bar\vx + \gvt,
\end{equation*}
which establishes translation equivariance, and thus shows that the CoM initialization of DistEGNN is E(3)-equivariant.
\end{proof}
}

\rev{
\begin{theorem}[\cref{prop:equi_distegnn}]\label{th:diste3}
The message passing process defined in \cref{alg:dist} of DistEGNN is E(3)-equivariant.
\end{theorem}
\begin{proof}
According to \cref{prop:equofnn}, it suffices to show the equivariance of each module. Since part of DistEGNN is reused from FastEGNN, we only need to prove the E(3)-equivariance (or E(3)-invariance) of the newly introduced modules (\emph{i.e.} \cref{eq:upd_z_dist,eq:upd_s_dist}). We first prove the equivariance of \cref{eq:upd_z_dist}:
\begin{equation*}
    \gvz_c^{(l+1)} = \gvz_c^{(l)} + \frac{1}{N} \sum_{d=1}^{D}\sum_{n_d=1}^{N_{d}} (\gvz_c^{(l)}-\gvx_{n_d}^{(l)}) \varphi_{Z}(\vm_{n_dc}^{v}).
\end{equation*}
It is clear that $\frac{1}{N} \sum_{d=1}^{D}\sum_{n_d=1}^{N_{d}} (\gvz_c^{(l)}-\gvx_{n_d}^{(l)}) \varphi_{Z}(\vm_{n_dc}^{v})$ is O(3)-equivariant and translation invariant; therefore, the entire update is E(3)-equivariant. Moreover, we show the invariance of \cref{eq:upd_s_dist}:
\begin{equation*}
    \vs_c^{(l+1)}=\vs_c^{(l)}+\varphi_{S}\left(\vs_c^{(l)},\frac{1}{N} \sum_{d=1}^{D}\sum_{n_d=1}^{N_{d}}\vm_{n_dc}^{v} \right).
\end{equation*}
Since all variables in this expression are scalars, this module is E(3)-invariant.
\end{proof}
}
\newpage
\subsection{Generality Analysis}\label{appendix:generality_analysis}

To examine the general applicability of our proposed virtual node mechanism, we introduce it to three representative geometric GNNs: RF~\citep{kohler2019equivariant}, an equivariant model; SchNet~\citep{schutt2018schnet}, an invariant GNN; and TFN~\citep{thomas2018tensor}, a high-degree steerable model that leverages spherical harmonics for equivariant representation.

\subsubsection{Implementation Details}

\textbf{FastRF.} RF is a scalarization-based equivariant model that shares a similar design with EGNN. However, it differs in that it computes messages solely based on node distances, without incorporating node features into message construction process. Accordingly, in implementing FastRF, we remove the node features $\vh$ from FastEGNN and also exclude the virtual node features $\mS$ to maintain consistency with RF’s design.

\textbf{FastSchNet.} SchNet is an invariant geometric graph neural network. To update the spatial coordinates, we incorporate an equivariant coordinate prediction module that takes the invariant features generated by SchNet as input. 
Besides, real nodes receive additional messages from virtual nodes, and the overall update rule becomes:
\begin{equation}
\begin{aligned}
    \gvx_i^{(l+1)}=&\underbrace{\left[\gvx_i^{(l)} +\sum_{j\in\mathcal{N}_i}(\gvx_{i}^{(l)}-\gvx_{j}^{(l)})\varphi(\vh_i^{(l)},\vh_j^{(l)},\ve_{ij})\right]}_{\texttt{update by SchNet}} \\
    &+ \left[\frac{1}{C}\sum_{c=1}^C(\gvx_i^{(l)}-\gvz^{(l)}_c)\varphi_{x}^v(\vm_{ic}^v)+\varphi_v(\vh_i^{(l)})\gvv_i^{(0)}\right].
\end{aligned}
\end{equation}

\textbf{FastTFN.} Unlike the two models above, TFN is an SO(3)-equivariant model that leverages high-degree tensors and spherical harmonics to model rotational symmetries. To implement FastTFN, we first reduce the number of the channels used in the original TFN to be 1, yielding single-channel TFN, which alleviates computational overhead and is found to be more compatible with virtual node learning. In each layer of FastTFN, the updated coordinates of real nodes are first computed using a TFN layer and then further refined by aggregating messages from virtual nodes via the following formulation:
\begin{equation}
\begin{aligned}
\gvx_i^{(l+1)}=&\underbrace{\left[\gvx_i^{(l)}+\sum_{j\in\mathcal{N}_i}\sY^{(\sL)}(\hat{\vx}_{ij}^{(l)})\otimes_{\text{cg}}^\sW\left(\vh_j^{(l)}\|\gvv_{j}^{(0)}\right)\right]}_{\texttt{update by TFN}} \\
    &+\left[\frac{1}{C}\sum_{c=1}^C(\gvx^{(l)}_i-\gvz^{(l)}_c)\varphi_{x}^v(\vm_{ic}^v)+\varphi_v(\vh_i^{(l)})\gvv_i^{(0)}\right],
\end{aligned}
\end{equation}
where the $\gvx_i^{(l)}$ is the coordinate of real node $i$, and $\sY^{\sL}(\cdot)$ embeds the normalized difference $\hat{\vx}_{ij}^{(l)}\coloneqq (\gvx^{(l)}_i-\gvx^{(l)}_j) / \|\gvx^{(l)}_i-\gvx^{(l)}_j\|$ between real node $i$ and $j$ into spherical harmonics with type in set $\sL=\{0, 1, 2,\dots,\texttt{max\_degree}\}$. $\otimes_{\text{cg}}^\sW(\cdot, \cdot)$ indicates Clebsch-Gordan (CG) tensor product with only type-1 outputs, the set $\sW$ contains the corresponding learnable weights~\cite{han2024survey}, $\vh_j$ and $\gvv_j$ are features (type-0) and velocity (type-1) of real node $j$ respectively.

\subsubsection{Experiments}

\begin{table*}[ht!]
\centering
\caption{Comparisons between EGNN~\citep{satorras2021en}, RF~\citep{kohler2019equivariant}, SchNet~\citep{schutt2018schnet}, TFN~\citep{thomas2018tensor} and their enhanced models involving virtual node learning. Note that the TFN series does not support the edge dropping rate of 1.00, as the initialization via spherical harmonics requires the presence of edges.}
\label{tab:generality_analysis}
\adjustbox{width=\textwidth, center, padding = 2pt}{
\begin{tabular}{
    p{0.11\textwidth}<{\raggedright}
    p{0.078\textwidth}<{\centering}
    p{0.078\textwidth}<{\centering}
    p{0.078\textwidth}<{\centering}
    p{0.003\textwidth}<{\centering}
    p{0.078\textwidth}<{\centering}
    p{0.078\textwidth}<{\centering}
    p{0.078\textwidth}<{\centering}
    p{0.003\textwidth}<{\centering}
    p{0.078\textwidth}<{\centering}
    p{0.078\textwidth}<{\centering}
    p{0.078\textwidth}<{\centering}
}
\toprule
                  & \multicolumn{3}{c}{$N$-body System ($\times 10^{-2}$)} & & \multicolumn{3}{c}{Protein Dynamics} & & \multicolumn{3}{c}{Water-3D ($\times 10^{-4}$)} \\
    Dropping Rate & $0.00$ & $0.75$ & $1.00$ &                               & $0.00$ & $0.75$ & $1.00$ &             & $0.00$ & $0.75$ & $1.00$                        \\
\midrule
    EGNN~\citep{satorras2021en}      & $1.30$ & $1.34$ & $11.62$ &           & $2.25$ & $2.24$ & $2.26$ &             & $6.00$ & $9.87$ & $12.38$                       \\
    FastEGNN-$1$                     & $1.04$ & $1.03$ & $9.72$  &           & $1.89$ & $1.91$ & $2.00$ &             & $3.25$ & $3.64$ & $4.36$                        \\ 
    FastEGNN-$3$                     & $1.06$ & $1.09$ & $9.52$  &           & $1.82$ & $1.85$ & $1.91$ &             & $2.58$ & $3.13$ & $3.60$                        \\
    FastEGNN-$10$                    & $1.10$ & $0.99$ & $9.25$  &           & $1.84$ & $1.88$ & $1.99$ &             & $2.66$ & $2.94$ & $3.40$                        \\ 
\midrule
    RF~\citep{kohler2019equivariant} & $5.76$ & $5.77$ & $11.65$ &           & $2.25$ & $2.26$ & $2.26$ &             & $12.94$ & $12.10$ & $12.49$                     \\
    FastRF-$1$                       & $4.17$ & $4.19$ & $10.22$ &           & $2.18$ & $2.21$ & $2.21$ &             & $7.33$  & $7.31$  & $7.61$                      \\
    FastRF-$3$                       & $4.13$ & $4.90$ & $10.24$ &           & $2.19$ & $2.21$ & $2.21$ &             & $7.71$  & $7.35$  & $7.97$                      \\
    FastRF-$10$                      & $4.04$ & $4.24$ & $10.20$ &           & $2.19$ & $2.21$ & $2.21$ &             & $7.72$  & $7.90$  & $8.04$                      \\
\midrule
    SchNet~\citep{schutt2018schnet}  & $24.83$ & $25.80$ & $35.49$ &         & $2.56$ & $2.56$ & $2.60$ &             & $35.02$ & $52.73$ & $57.81$                     \\ 
    FastSchNet-$1$                   & $23.61$ & $23.70$ & $25.96$ &         & $2.01$ & $2.03$ & $2.22$ &             & $14.29$ & $15.23$ & $18.69$                     \\
    FastSchNet-$3$                   & $23.60$ & $23.59$ & $25.95$ &         & $1.98$ & $2.08$ & $2.14$ &             & $13.36$ & $13.16$ & $15.14$                     \\
    FastSchNet-$10$                  & $23.73$ & $23.80$ & $25.97$ &         & $1.99$ & $2.05$ & $2.11$ &             & $14.38$ & $13.09$ & $13.94$                     \\
\midrule
    TFN~\citep{thomas2018tensor}     & $1.62$ & $4.53$ & --      &           & $2.25$ & $2.26$ & --     &             & $9.10$ & $15.91$ & --                           \\
    FastTFN-$1$                      & $3.42$ & $4.34$ & --      &           & $1.90$ & $1.99$ & --     &             & $3.89$ & $5.15$  & --                           \\
    FastTFN-$3$                      & $3.42$ & $3.65$ & --      &           & $1.86$ & $1.92$ & --     &             & $3.33$ & $3.98$  & --                           \\
    FastTFN-$10$                     & $3.87$ & $3.99$ & --      &           & $1.92$ & $1.95$ & --     &             & $3.15$ & $3.61$  & --                           \\
\bottomrule
\end{tabular}
}
\end{table*}

\begin{figure*}[ht!]
    \centering
    \includegraphics[width=\linewidth]{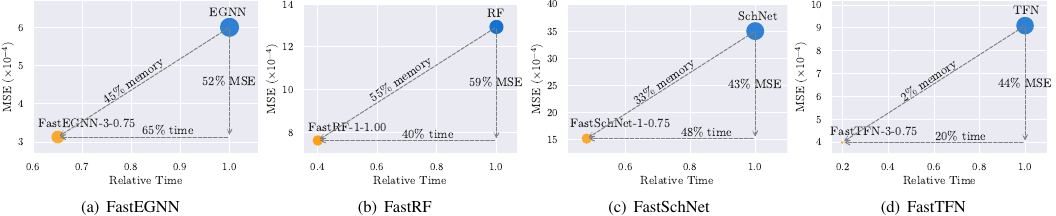}
    \vspace{-18pt}
    \caption{The performance of four base and enhanced models on the Water-3D dataset in terms of speed, accuracy, and memory. Points closer to the lower-left corner and smaller in size indicate better overall efficiency. }
    \label{fig:fast_models}
\end{figure*}

We evaluate the performance of baseline and enhanced models on $N$-body, Protein Dynamics, and Water-3D dataset under the same setups and metrics as described in~\cref{sec:exp_effectiveness}. Each enhanced model includes 9 variants, constructed by combining three different numbers of virtual nodes $\{1, 3, 10\}$  with three edge dropping rates $\{0.00, 0.75, 1.00\}$. The results are detailed in~\cref{tab:generality_analysis}.

We further visualize the improvements in speed, accuracy, and memory usage of the four enhanced models compared to their corresponding base models, as shown in~\cref{fig:fast_models}. We can observe that, by introducing virtual nodes and applying edge dropping, all four enhanced models achieve significant improvements in both accuracy and efficiency. For example, the FastRF-1-1.00 achieves 59\% of RF’s MSE while using only 40\% of the training time and 55\% of the memory. Moreover, compared to the resource-intensive TFN, FastTFN uses only 2\% of the memory and 20\% of the computation time, while reducing prediction error by 56\%. These examples consistently demonstrate the advantages of our approach.

Of all 36 experimental configurations, the only unexpected case where the baseline model outperforms its enhanced version is TFN on the $N$-body dataset without edge dropping. One possible reason is that, during the construction of FastTFN, we reduced the number of channels in the TFN layer from 32 to 1 for lower computational cost and better compatibility, which in turn weakened its ability to simulate the $N$-body dataset. However, we find that TFN is highly sensitive to graph connectivity. When the edge dropping rate increases to 0.75, the MSE sharply rises from 1.62 to 4.53. In contrast, FastTFN exhibits greater robustness in such settings. With 75\% of the edges dropped, FastTFN reverses the performance gap and achieves better accuracy than TFN. This demonstrating that the virtual node mechanism effectively mitigates the impact of missing connections and enhances the model’s expressiveness and stability on sparse graphs.

\newpage
\subsection{Experiment Details}
\label{appendix:experiment_details}
\subsubsection{Dataset Details}
\begin{table}[H]
\centering
\caption{Basic statistics of the four datasets, with graph sizes increasing progressively. Among them, Fluid113K is our generated large-scale fluid simulation dataset, where each graph contains over 100K nodes and an average of 1.7M edges.}
\label{tab:dataset}
\renewcommand\arraystretch{1.25}
\adjustbox{width=\textwidth, center}{
\begin{tabular}{
    m{0.28\textwidth}<{\raggedright}
    m{0.15\textwidth}<{\centering}
    m{0.15\textwidth}<{\centering}
    m{0.15\textwidth}<{\centering}
    m{0.15\textwidth}<{\centering}
}
\toprule
                                             & $N$-body System           & Protein Dynamics       & Water-3D                    & Fluid113K             \\
\midrule
    Number of Samples (Train / Valid / Test) & $5,000 / 2,000 / 2,000$   & $2,481 / 827 / 863$    & $15,000 / 1,500 / 1,500$    & $1,600 / 320 / 320$   \\
    Number of Nodes ($N$)                    & $100$                     & $855$                  & $7,806$ (Avg)               & $113,140$ (Avg)       \\
    Number of Edges ($|\gE|$)                & $9,900$                   & $55,107$ (Avg)         & $94,931$ (Avg)              & $1,706,973$ (Avg)     \\
    Default Radius ($r$)                     & $\infty$                  & $10$\AA                & $0.035$                     & $0.075$               \\
    Predict Time Interval ($\Delta t$)       & $10$                      & $15$                   & $15$                        & $20$                  \\
\bottomrule
\end{tabular}
}
\end{table}

\subsubsection{Hyper Parameters}\label{sec:implementation_details}
\begin{table}[H]
\centering
\caption{Default hyper-parameters on four datasets}
\label{tab:hyper-parameter}
\renewcommand\arraystretch{1.25}
\adjustbox{width=\textwidth, center, padding = 2pt}{
\begin{tabular}{
    m{0.28\textwidth}<{\raggedright}
    m{0.15\textwidth}<{\centering}
    m{0.15\textwidth}<{\centering}
    m{0.15\textwidth}<{\centering}
    m{0.15\textwidth}<{\centering}
}
\toprule
    Hyperparameter                                 & $N$-body System   & Protein Dynamics   & Water-3D         & Fluid-113K          \\
\midrule 
    Bandwidth $\sigma $ in \cref{eq:mmd}           & $1.5$             & $1.0$              & $1.5$            & $3.0$                 \\
    Balancing factor $\lambda$ in \cref{eq:loss}   & $0.03$            & $0.50$             & $0.01$           & $0.01$              \\
    Number of Sampled Nodes $N$ in \cref{eq:mmd}   & $3$               & $3$                & $3$              & $50$                \\
    Number of Virtual Channel $C$                  & $\{1, 3, 10\}$    & $\{1, 3, 10\}$     & $\{1, 3, 10\}$   & $\{1, 3, 10, 20\}$            \\
\midrule
    Learning Rate                                  & $5\text{e-}4$     & $5\text{e-}4$      & $5\text{e-}4$    & $5\text{e-}4^*$       \\
    Weight Decay                                   & $1\text{e-}12$    & $1\text{e-}12$     & $1\text{e-}12$   & $1\text{e-}12$      \\
    Epochs                                         & $2500$            & $2000$             & $1000$           & $2500$              \\
    Early Stop                                     & $200$             & $100$              & $100$            & $100$               \\
\midrule
    Number of Layers                               & $4$               & $4$                & $4$              & $4$                 \\
    Hidden Dimension                               & $64$              & $64$               & $64$             & $64$                \\
\bottomrule
\multicolumn{5}{l}{\footnotesize *For single-device training on Fluid113K, the learning rate needs to be reduced to $5\text{e-}5$ for stability in early training.}
\end{tabular}
}
\end{table}

\subsubsection{\rev{Implementation of Graph Partition Strategies}}
\label{appendix:graph_partition}

\rev{
We explore four partition strategies, including random partition, spectral clustering, and partition using K-Means and METIS. The implementation details are described below. 
}
\rev{
\begin{enumerate}
    \item The random partition assigns each node to a device with equal probability by default. This approach is simple and flexible. It enables dynamic adjustment of assignment probabilities based on the computational capacity of different devices to balance the computational load and memory overhead.
    \item The K-Means partition clusters spatially close nodes together by minimizing the squared Euclidean distance within each cluster.
    \item The spectral clustering uses the eigenvectors of the graph Laplacian to map graph connectivity into a continuous embedding space, and leverages K-Means in this space to obtain the final clustering results. In practice, we construct a weighted graph using an RBF kernel over the node positions, on which the spectral clustering is performed.
    \item The METIS adopts a multilevel framework that progressively coarsens the original large graph, computes a high-quality partition on the coarsened graph, and then successively refines the solution back to the original graph. Its objective is to minimize cross-partition edges while balancing the number of nodes across partitions.
\end{enumerate}
}

\rev{Among the explored partition strategies, random and K-Means perform node-level partition, while METIS and spectral methods require a pre-built graph and aim to minimize cross-partition edges. Notably, spectral partition and K-Means do not ensure balanced node counts across partitions, which can negatively impact computational efficiency and peak memory consumption in standard multi-device parallel settings.}

\newpage
\subsection{More Experimental Results}

\subsubsection{\revii{Extended Experiments of DistEGNN on Random Partition Strategies}}

\revii{Here, we present detailed experimental results of DistEGNN under a random graph partition strategy, as reported in~\cref{tab:distribute_water3d_random_same_radius,tab:distribute_largefluid_random_same_radius}. Compared to the structure-aware METIS partition used in the main experiments, the random partition produces more scattered subgraphs with fewer preserved edges, leading to higher speedups and lower memory overhead. At the same time, it places greater demands on the model's ability to compensate for cross-device message passing.}

\begin{table*}[ht!]
\centering
\ifrevisiontwo\color{red}\fi
\caption{\rev{Performance of EGNN~\citep{satorras2021en} and DistEGNN on the Water-3D dataset using the fixed cutoff radius for graph construction.}}
\label{tab:distribute_water3d_random_same_radius}
\adjustbox{width=\textwidth, center, padding = 2pt}{
\begin{tabular}{
    p{0.19\textwidth}<{\raggedright}
    p{0.115\textwidth}<{\raggedright}
    p{0.075\textwidth}<{\centering}
    p{0.075\textwidth}<{\centering}
    p{0.075\textwidth}<{\centering}
    p{0.075\textwidth}<{\centering}
    p{0.075\textwidth}<{\centering}
    p{0.075\textwidth}<{\centering}
    p{0.075\textwidth}<{\centering}
    p{0.075\textwidth}<{\centering}
}
\toprule
                                                  &                             &  $1$-device & $2$-devices & $3$-devices & $4$-devices & $5$-devices & $6$-devices & $7$-devices & $8$-devices   \\
\midrule
    \multirow{3}{*}{Graph information}            & Cutoff radius               &  $0.035$    & $0.035$     & $0.035$     & $0.035$     & $0.035$     & $0.035$     & $0.035$     & $0.035$     \\
                                                  & Edge count                  &  $94931$    & $23711$     & $10548$     & $5930$      & $3794$      & $2637$      & $1936$      & $1481$      \\
                                                  & Average degree              &  $12.16$    & $3.04$      & $1.35$      & $0.76$      & $0.49$      & $0.34$      & $0.25$      & $0.19$      \\
\midrule
    \multirow{5}{*}{MSE ($\times 10^{-4}$)}       & EGNN~\citep{satorras2021en} &  $6.00$     & $8.04$      & $8.98$      & $10.16$     & $10.95$     & $10.57$     & $11.03$     & $11.05$     \\
                                                  & DistEGNN-$1$                &  $3.25$     & $3.49$      & $3.71$      & $4.03$      & $4.00$      & $3.87$      & $4.11$      & $3.85$      \\
                                                  & DistEGNN-$3$                &  $2.58$     & $3.23$      & $3.26$      & $3.46$      & $3.47$      & $3.19$      & $3.45$      & $3.27$      \\
                                                  & DistEGNN-$10$               &  $2.66$     & $3.04$      & $3.03$      & $3.35$      & $3.34$      & $3.15$      & $3.46$      & $3.12$      \\
                                                  & \rev{DistEGNN-$20$}         & \rev{$2.66$}& \rev{$3.10$}& \rev{$3.05$}& \rev{$3.42$}& \rev{$3.41$}& \rev{$3.22$}& \rev{$3.73$}& \rev{$3.30$}      \\
\midrule 
    \multirow{5}{*}{Relative inference time}      & EGNN~\citep{satorras2021en} &  $1.00$     & $0.37$      & $0.21$      & $0.18$      & $0.17$      & $0.16$      & $0.17$      & $0.17$      \\
                                                  & DistEGNN-$1$                &  $1.35$     & $0.72$      & $0.54$      & $0.52$      & $0.48$      & $0.47$      & $0.47$      & $0.46$      \\
                                                  & DistEGNN-$3$                &  $1.51$     & $0.79$      & $0.60$      & $0.53$      & $0.49$      & $0.48$      & $0.49$      & $0.49$      \\
                                                  & DistEGNN-$10$               &  $1.81$     & $0.98$      & $0.73$      & $0.62$      & $0.56$      & $0.54$      & $0.53$      & $0.52$      \\
                                                  & \rev{DistEGNN-$20$}         & \rev{$2.40$}& \rev{$1.19$}& \rev{$0.86$}& \rev{$0.72$}& \rev{$0.64$}& \rev{$0.60$}& \rev{$0.57$}& \rev{$0.55$}      \\
\midrule
    \multirow{5}{*}{Device memory usage (GB)}     & EGNN~\citep{satorras2021en} &  $74.15$    & $17.23$     & $5.63$      & $2.60$      & $1.90$      & $1.14$      & $0.68$       & $0.50$       \\
                                                  & DistEGNN-$1$                &  $75.66$    & $17.93$     & $6.11$      & $3.05$      & $2.12$      & $1.34$      & $1.00$       & $0.76$       \\
                                                  & DistEGNN-$3$                &  $78.09$    & $19.32$     & $6.42$      & $4.31$      & $3.36$      & $2.45$      & $1.70$       & $1.40$       \\
                                                  & DistEGNN-$10$               &  $84.26$    & $27.63$     & $13.30$     & $10.66$     & $10.63$     & $7.15$      & $7.31$       & $4.89$       \\
                                                  & \rev{DistEGNN-$20$}         &\rev{$86.75$}&\rev{$63.87$}&\rev{$29.46$}&\rev{$22.96$}&\rev{$20.53$}&\rev{$17.41$}&\rev{$17.75$} &\rev{$12.82$}      \\
\bottomrule
\end{tabular}
}
\end{table*}

\begin{table*}[ht!]
\centering
\ifrevisiontwo\color{red}\fi
\caption{\rev{Performance of EGNN~\citep{satorras2021en} and DistEGNN on Fluid113K dataset using the fixed cutoff radius for graph construction.}}
\label{tab:distribute_largefluid_random_same_radius}
\adjustbox{width=\textwidth, center, padding = 2pt}{
\begin{tabular}{
    p{0.19\textwidth}<{\raggedright}
    p{0.115\textwidth}<{\raggedright}
    p{0.075\textwidth}<{\centering}
    p{0.075\textwidth}<{\centering}
    p{0.075\textwidth}<{\centering}
    p{0.075\textwidth}<{\centering}
    p{0.075\textwidth}<{\centering}
    p{0.075\textwidth}<{\centering}
    p{0.075\textwidth}<{\centering}
    p{0.075\textwidth}<{\centering}
}
\toprule
                                                  &                             &  $1$-device & $2$-devices & $3$-devices & $4$-devices & $5$-devices & $6$-devices & $7$-devices & $8$-devices   \\
\midrule
    \multirow{3}{*}{Graph information}            & Cutoff radius               &  $0.075$    & $0.075$     & $0.075$     & $0.075$     & $0.075$     & $0.075$     & $0.075$     & $0.075$       \\
                                                  & Edge count                  &  $1706973$  & $426430$    & $189639$    & $106615$    & $68233$     & $47387$     & $34851$     & $26676$       \\
                                                  & Average degree              &  $15.09$    & $3.76$      & $1.68$      & $0.94$      & $0.60$      & $0.41$      & $0.31$      & $0.24$        \\
\midrule
    \multirow{5}{*}{MSE ($\times 10^{-2}$)}       & EGNN~\citep{satorras2021en} &  $17.15$    & $22.22$     & $24.79$     & $26.26$     & $27.50$     & $28.34$     & $29.43$     & $29.02$       \\
                                                  & DistEGNN-$5$                &  $8.63$     & $9.22$      & $9.23$      & $9.50$      & $10.46$     & $10.06$     & $10.08$     & $9.98$        \\
                                                  & \rev{DistEGNN-$10$}         & \rev{$-$}   & \rev{$9.13$}& \rev{$9.03$}& \rev{$9.58$}& \rev{$9.96$}& \rev{$9.37$}& \rev{$9.94$}& \rev{$10.19$}       \\
                                                  & \rev{DistEGNN-$20$}         & \rev{$-$}   & \rev{$8.78$}& \rev{$8.98$}& \rev{$9.20$}& \rev{$10.00$}& \rev{$9.33$}& \rev{$9.86$}& \rev{$9.48$}       \\
                                                  & \rev{DistEGNN-$40$}         & \rev{$-$}   & \rev{$8.96$}& \rev{$9.35$}& \rev{$9.87$}& \rev{$9.93$}& \rev{$9.77$}& \rev{$10.70$}& \rev{$10.49$}    \\
\midrule 
    \multirow{5}{*}{Relative inference time}      & EGNN~\citep{satorras2021en} & $1.00$      & $0.33$      & $0.22$      & $0.19$      & $0.15$      & $0.15$      & $0.16$      & $0.17$      \\
                                                  & DistEGNN-$5$                & $1.64$      & $0.69$      & $0.48$      & $0.42$      & $0.36$      & $0.35$      & $0.35$      & $0.36$      \\
                                                  & \rev{DistEGNN-$10$}         & \rev{$2.00$}& \rev{$0.71$}& \rev{$0.53$}& \rev{$0.44$}& \rev{$0.38$}& \rev{$0.35$}& \rev{$0.35$}& \rev{$0.37$}      \\
                                                  & \rev{DistEGNN-$20$}         & \rev{$2.45$}& \rev{$0.97$}& \rev{$0.71$}& \rev{$0.57$}& \rev{$0.48$}& \rev{$0.44$}& \rev{$0.42$}& \rev{$0.41$}      \\
                                                  & \rev{DistEGNN-$40$}         & \rev{$3.75$}& \rev{$1.59$}& \rev{$1.12$}& \rev{$0.87$}& \rev{$0.74$}& \rev{$0.64$}& \rev{$0.59$}& \rev{$0.56$}      \\
\midrule
    \multirow{5}{*}{Device memory usage (GB)}     & EGNN~\citep{satorras2021en} & $41.55$     & $15.00$     & $6.20$      & $4.00$      & $2.02$      & $1.22$      & $0.69$      & $0.47$     \\
                                                  & DistEGNN-$5$                & $47.97$     & $15.89$     & $9.37$      & $4.19$      & $3.09$      & $3.02$      & $2.11$      & $1.58$     \\
                                                  & \rev{DistEGNN-$10$}         &\rev{$55.88$}&\rev{$17.22$}&\rev{$8.83$} & \rev{$6.55$}& \rev{$5.74$}& \rev{$4.03$}& \rev{$3.35$}& \rev{$2.93$}      \\
                                                  & \rev{DistEGNN-$20$}         &\rev{$73.25$}&\rev{$27.21$}&\rev{$11.97$}&\rev{$11.63$}&\rev{$9.98$} & \rev{$7.66$}& \rev{$6.62$}& \rev{$5.66$}      \\
                                                  & \rev{DistEGNN-$40$}         &\rev{$95.07$}&\rev{$53.10$}&\rev{$21.10$}&\rev{$20.93$}&\rev{$20.37$}&\rev{$15.03$}&\rev{$14.75$}&\rev{$13.30$}      \\
\bottomrule
\end{tabular}
}
\end{table*}

\revii{The results also clearly show that, as the number of distributed devices increases, the accuracy of EGNN declines more significantly than that of DistEGNN. This is consistent with the conclusion in the main experiments and further demonstrates the effectiveness of our virtual nodes in facilitating distributed communication. We also experiment with dynamically adjusting the cutoff radius under random partition to maintain a constant number of edges, with results reported in~\cref{tab:distribute_water3d_random_same_degree}. We observe that using a fixed radius in the distributed setting achieves higher speedup, while the dynamic radius setting yields better predictive accuracy.}

\begin{table*}[ht!]
\centering
\ifrevisiontwo\color{red}\fi
\caption{Performance of EGNN~\citep{satorras2021en} and DistEGNN on Water-3D dataset with dynamic cutoff raduis for graph construction.}
\label{tab:distribute_water3d_random_same_degree}
\adjustbox{width=\textwidth, center, padding = 2pt}{
\begin{tabular}{
    p{0.19\textwidth}<{\raggedright}
    p{0.115\textwidth}<{\raggedright}
    p{0.075\textwidth}<{\centering}
    p{0.075\textwidth}<{\centering}
    p{0.075\textwidth}<{\centering}
    p{0.075\textwidth}<{\centering}
    p{0.075\textwidth}<{\centering}
    p{0.075\textwidth}<{\centering}
    p{0.075\textwidth}<{\centering}
    p{0.075\textwidth}<{\centering}
}
\toprule
                                                  &                             &  $1$-device & $2$-devices & $3$-devices & $4$-devices & $5$-devices & $6$-devices & $7$-devices & $8$-devices   \\
\midrule
    \multirow{3}{*}{Graph information}            & Cutoff radius               &  $0.035$    & $0.045$     & $0.052$     & $0.057$     & $0.062$     & $0.066$     & $0.070$     & $0.074$     \\
                                                  & Edge count                  &  $94931$    & $95535$     & $94155$     & $93576$     & $95576$     & $93137$     & $93276$     & $95576$     \\
                                                  & Average degree              &  $12.16$    & $12.24$     & $12.06$     & $11.99$     & $12.24$     & $11.93$     & $11.95$     & $12.24$     \\
\midrule
    \multirow{3}{*}{MSE ($\times 10^{-4}$)}       & DistEGNN-$1$                &  $3.25$     & $3.18$      & $3.47$      & $3.30$      & $3.43$      & $3.32$      & $3.33$      & $3.27$      \\
                                                  & DistEGNN-$3$                &  $2.58$     & $2.76$      & $3.01$      & $2.93$      & $3.15$      & $3.10$      & $2.96$      & $3.03$      \\
                                                  & DistEGNN-$10$               &  $2.66$     & $2.83$      & $3.04$      & $2.88$      & $3.16$      & $2.89$      & $2.94$      & $2.96$      \\
\midrule
\multirow{3}{*}{Relative inference time}          & DistEGNN-$1$                & $1.35$      & $0.93$      & $0.76$      & $0.68$      & $0.62$      & $0.58$      & $0.56$      & $0.57$      \\
                                                  & DistEGNN-$3$                & $1.51$      & $0.97$      & $0.79$      & $0.72$      & $0.64$      & $0.60$      & $0.57$      & $0.58$      \\
                                                  & DistEGNN-$10$               & $1.81$      & $1.18$      & $0.93$      & $0.80$      & $0.71$      & $0.66$      & $0.64$      & $0.62$      \\
\midrule
    \multirow{3}{*}{Device memory usage (GB)}     & DistEGNN-$1$                & $75.66$     & $35.61$     & $21.42$     & $15.02$     & $13.78$     & $7.57$      & $8.53$       & $6.54$       \\
                                                  & DistEGNN-$3$                & $78.09$     & $36.71$     & $22.21$     & $15.60$     & $14.29$     & $8.02$      & $8.85$       & $6.84$       \\
                                                  & DistEGNN-$10$               & $84.26$     & $38.33$     & $22.76$     & $15.48$     & $15.13$     & $8.98$      & $10.62$      & $6.93$       \\

\bottomrule
\end{tabular}
}
\end{table*}

\revii{In conclusion, DistEGNN is highly adaptable to different graph partition strategies. We release the source code to enable users to flexibly design and register customized partition strategies according to specific tasks and data types, enabling effective trade-offs between accuracy and computational efficiency.}

\newpage

\subsubsection{\rev{Incorporation of Device-Level Attention Mechanism}}
\label{appendix:deivce_attention}
\rev{
We design a new variant of DistEGNN that incorporates a device-level attention mechanism. This mechanism allows the virtual node to assign distinct attention weights to each device's subgraph representation. Specifically, we first compute the average distance between the virtual node and the real nodes residing on each device. This average distance is then normalized and used to derive a weight coefficient, which captures the relative attention that the virtual node allocates to each device. We apply this weight to the real-virtual messages to enable adaptive information aggregation. The detailed formulation is given as follows:
\begin{gather} 
    \va^{(l)}_{c} = \left[\frac{1}{N_d}\sum_{n_d=1}^{N_{d}} \|\gvx^{(l)}_{n_d} - \gvz_c^{(l)}\|\right]_{d=1}^D, \\
    \valpha_{c}^{(l)} =  \text{softmax}( \text{normalize} (\lambda \va^{(l)}_{c})) \triangleq \left[ \alpha_{cd} \right]_{d=1}^{D},  \\
    \vm_{n_dc}^{v} = \alpha_{cd}^{(l)} \cdot \vm_{n_dc}^{v},
\end{gather}
where $\lambda$ is a learnable parameter and $\alpha_{cd}$ denotes the attention weight of the virtual node $c$ over the subgraph on device $d$. The weight $\alpha_{cd}$ is applied to the real virtual message $\vm_{n_dc}^{v}$, which is then used in the computation of 
~\cref{eq:xr(l+1),eq:hr(l+1),eq:upd_z_dist,eq:upd_s_dist}.
}

\begin{table*}[h!]
\centering
\ifrevision\color{red}\fi
\caption{Comparison of MSE before and after incorporating attention from virtual nodes to each device under METIS graph partitioning.}
\label{tab:device_attn}
\adjustbox{width=\textwidth, center, padding = 2pt}{
\begin{tabular}{
    p{0.10\textwidth}<{\raggedright}
    p{0.075\textwidth}<{\centering}
    p{0.075\textwidth}<{\centering}
    p{0.075\textwidth}<{\centering}
    p{0.075\textwidth}<{\centering}
    p{0.075\textwidth}<{\centering}
    p{0.075\textwidth}<{\centering}
    p{0.075\textwidth}<{\centering}
    p{0.075\textwidth}<{\centering}
}
\toprule
                                &  $1$-device & $2$-devices & $3$-devices & $4$-devices & $5$-devices & $6$-devices & $7$-devices & $8$-devices   \\
\midrule
    METIS                       &  $2.58$     & $2.75$      & $2.75$      & $2.72$      & $3.25$      & $2.95$      & $2.90$      & $2.80$        \\
    METIS + Attn                &  $2.58$     & $2.74$      & $2.81$      & $2.70$      & $3.15$      & $2.89$      & $2.85$      & $2.75$        \\  
\bottomrule
\end{tabular}
}

\end{table*}

\rev{
We evaluate DistEGNN-3's performance after incorporating the device-aware attention, using the METIS graph partition strategy on the Water-3D dataset. The results are reported in~\cref{tab:device_attn}. We observe that introducing attention weights between virtual nodes and each device yields a modest but consistent performance gain (except the case when using 3 devices), indicating that the attention mechanism allows the model to better capture relevant dependencies.
}

\subsubsection{\rev{Training Curves of FastEGNN and DistEGNN}}
\label{appendix:training_curves}

\begin{figure*}[ht!]
    \centering
    \ifrevision\color{red}\fi
    \includegraphics[width=\linewidth]{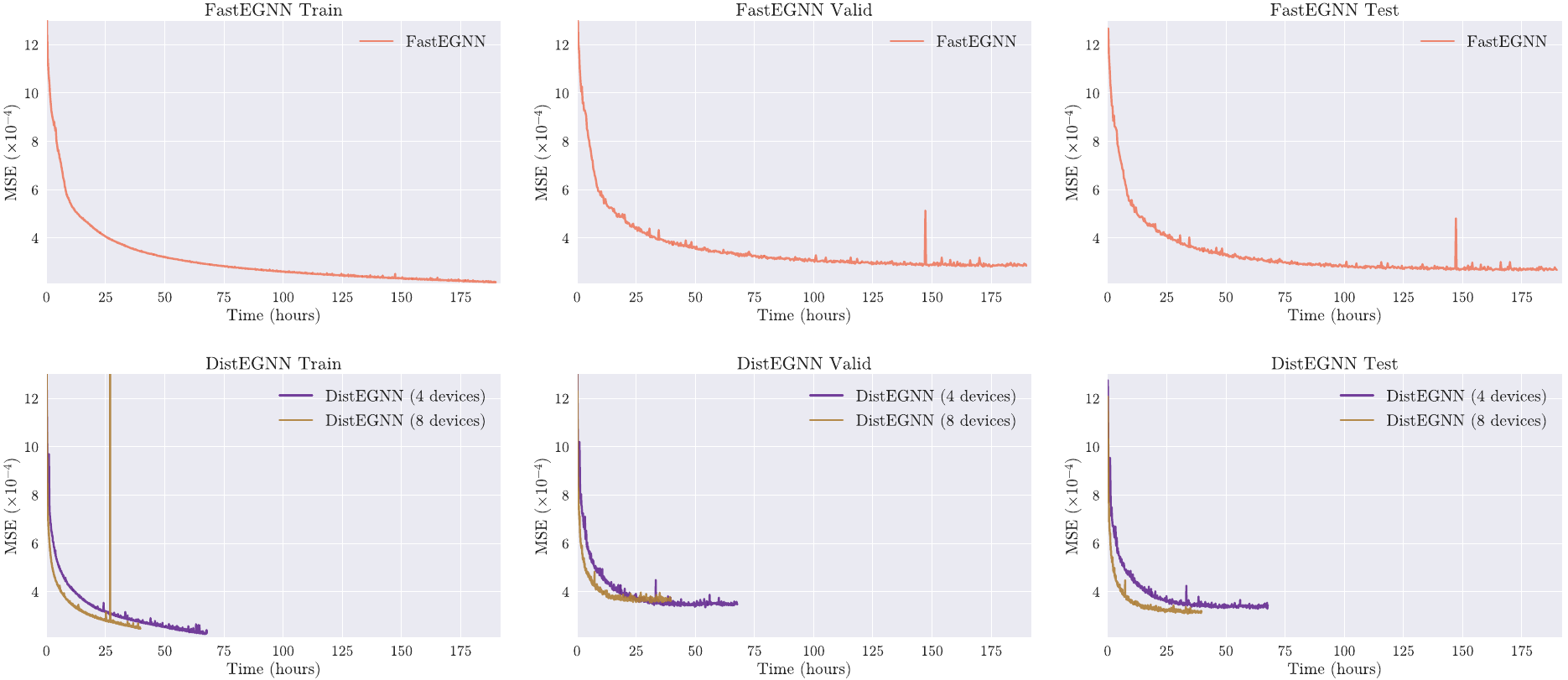}
    \vspace{-18pt}
    \caption{Training loss over time for FastEGNN and DistEGNN with 10 virtual nodes on the Water-3D dataset. In the top row, the orange curve corresponds to FastEGNN, while in the bottom row, the purple and brown curves correspond to DistEGNN-10 using 4 and 8 devices, respectively.}
    \label{fig:training_plot}
\end{figure*}

\rev{
We provide the training curves of FastEGNN and DistEGNN on the Water-3D dataset in~\cref{fig:training_plot}. Both models use 10 virtual nodes, and DistEGNN is trained on either 4 or 8 devices. The x‑axis represents training time, and the y‑axis denotes the Mean Squared Error (MSE). As shown in the figure, both variants of DistEGNN converge faster than FastEGNN while reaching a comparable final performance, demonstrating the benefit of multi‑device parallelism in DistEGNN. Furthermore, DistEGNN trained on 8 devices shows better efficiency than when using 4 devices. These results confirm the scalability of DistEGNN, supporting its application to large‑scale geometric graphs.
}

\end{document}